\documentclass[10pt,journal,compsoc]{IEEEtran}

\ifCLASSOPTIONcompsoc
  \usepackage[nocompress]{cite}
\else
  \usepackage{cite}
\fi

\usepackage{graphicx}%
\usepackage{subcaption}
\usepackage{multirow}%
\usepackage{amsmath,amssymb,amsfonts}%
\usepackage{amsthm}%
\usepackage{mathrsfs}%
\usepackage{xcolor}%
\usepackage{textcomp}%
\usepackage{manyfoot}%
\usepackage{booktabs}%
\usepackage{algorithm}%
\usepackage{algorithmicx}%
\usepackage{algpseudocode}%
\usepackage{listings}%
\usepackage{epstopdf}
\usepackage[capitalise]{cleveref}
\usepackage{bm}
\usepackage{url}
\usepackage{makecell}
\usepackage{tikz}

\newcommand*{\circled}[1]{\lower.7ex\hbox{\tikz\draw (0pt, 0pt)%
    circle (.5em) node {\makebox[1em][c]{\small #1}};}}

\newtheorem{theorem}{Theorem}
\newtheorem{definition}{Definition}

\newtheorem{lemma}{Lemma}

\ifCLASSINFOpdf

\else

\fi

\begin{document}

\title{Distilling the Unknown to Unveil Certainty}

\author{Zhilin Zhao,
        Longbing Cao,
        Yixuan Zhang,
        Kun-Yu Lin,
        and~Wei-Shi Zheng
\IEEEcompsocitemizethanks{

\IEEEcompsocthanksitem Zhilin Zhao is with the School of Computer Science and Engineering, Key Laboratory of Machine Intelligence and Advanced Computing, Ministry of Education, Sun Yat-sen University, Guangzhou 510275, China.\protect\\
E-mail: zhaozhlin@mail.sysu.edu.cn
\IEEEcompsocthanksitem Longbing Cao is with the School of Computing, Macquarie University, Sydney, NSW 2109, Australia.\protect\\
E-mail: longbing.cao@mq.edu.au.
\IEEEcompsocthanksitem Yixuan Zhang is with the Statistics and Data Science, Southeast University, Nanjing, 211189, Jiangsu, China.\protect\\
E-mail: yixuan.zhang@seu.edu.cn.
\IEEEcompsocthanksitem Kun-Yu Lin is with the School of Computer Science and Engineering, Sun Yat-sen University, Guangzhou 510275, China.\protect\\
E-mail: kunyulin14@outlook.com.
\IEEEcompsocthanksitem Wei-Shi Zheng is with the School of Computer Science and Engineering, Key Laboratory of Machine Intelligence and Advanced Computing, Ministry of Education, Sun Yat-sen University, Guangzhou 510275, China.\protect\\
E-mail: wszheng@ieee.org.
\IEEEcompsocthanksitem Corresponding author: Kun-Yu Lin. 
}
}


\IEEEtitleabstractindextext{%
\begin{abstract}
Out-of-distribution (OOD) detection is critical for identifying test samples that deviate from in-distribution (ID) data, ensuring network robustness and reliability. This paper presents a flexible framework for OOD knowledge distillation that extracts OOD-sensitive information from a network to develop a binary classifier capable of distinguishing between ID and OOD samples in both scenarios, with and without access to training ID data. To accomplish this, we introduce Confidence Amendment (CA), an innovative methodology that transforms an OOD sample into an ID one while progressively amending prediction confidence derived from the network to enhance OOD sensitivity. This approach enables the simultaneous synthesis of both ID and OOD samples, each accompanied by an adjusted prediction confidence, thereby facilitating the training of a binary classifier sensitive to OOD. Theoretical analysis provides bounds on the generalization error of the binary classifier, demonstrating the pivotal role of confidence amendment in enhancing OOD sensitivity. Extensive experiments spanning various datasets and network architectures confirm the efficacy of the proposed method in detecting OOD samples.
\end{abstract}

}

\maketitle

\IEEEdisplaynontitleabstractindextext

%
\IEEEpeerreviewmaketitle

\IEEEraisesectionheading{\section{Introduction}\label{sec:introduction}}
Deep neural networks, trained on samples referred to as in-distribution (ID), have shown remarkable generalization capabilities for test samples aligned with the same distribution~\cite{DBLP:conf/iclr/SuzukiAN20}. However, they struggle when encountering \textit{out-of-distribution} (OOD) samples derived from different distributions~\cite{10136820, DBLP:conf/nips/YangWZZDPWCLSDZ22}. Alarmingly, these networks are prone to assigning high-confidence predictions to such OOD samples, thereby blurring the critical distinction between ID and OOD samples~\cite{10271740}. This issue arises because standard training procedures do not impose constraints on how the network should react to OOD samples, resulting in distribution vulnerability~\cite{FIG:23}. In real-world applications, the inability to identify OOD samples can lead to severe consequences, emphasizing the critical importance of OOD detection.

For a standard network trained on ID samples, existing methodologies for detecting its OOD samples fall into two main categories~\cite{DBLP:journals/tmlr/SalehiMHLRS22}. The first relies on post-hoc analysis of the output from networks, without altering the original architecture or needing access to the original ID training data. These methods, however, are intrinsically limited by the sensitivity of the existing network to OOD samples. The second modifies the loss function and training process by incorporating OOD prior knowledge. While this enhances OOD sensitivity, it compromises the generalization capabilities for ID samples and necessitates retraining networks on the original ID training data, which may be impractical. This is due to data privacy laws and regulations that may restrict the sharing and reuse of sensitive or proprietary data, making it hard to access the original datasets for retraining purposes.

\begin{figure}
  \centering
  \includegraphics[width=0.47\textwidth]{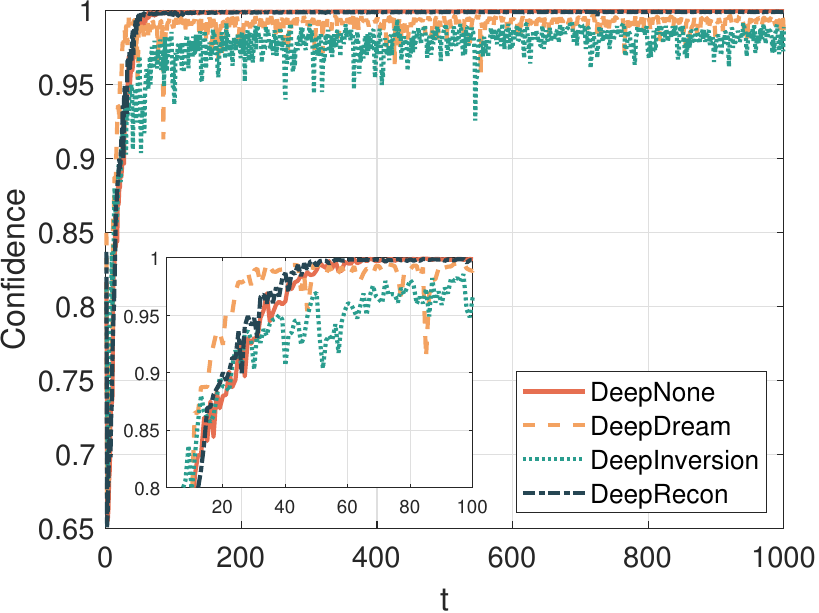}\\
  \caption{Variations in confidence levels for synthesized images over iterations using different methods. The standard network, built on a ResNet18 backbone~\cite{DBLP:conf/cvpr/HeZRS16}, is trained on the CIFAR10 dataset~\cite{CIFAR10:09}. DeepNone amplifies the confidence of noise by optimizing for cross-entropy loss using a random label, without any regularization constraints. In contrast, both DeepDream~\cite{mordvintsev2015deepdream} and DeepInversion~\cite{DBLP:conf/cvpr/YinMALMHJK20} apply additional regularizations to the synthesized samples. When original ID data is available, DeepRecon employs mean squared error to enhance confidence by more closely approximating the original ID data.}
  \label{fig:conf}
\end{figure}

To combine the strengths of these two existing algorithms and address their limitations, we propose a flexible framework for OOD knowledge distillation~\cite{DBLP:conf/iclr/Allen-ZhuL23, DBLP:journals/ijcv/GouYMT21, wu-etal-2023-multi}. This framework extracts OOD-sensitive knowledge from a standard network without altering its architecture or needing access to the training data. The extracted knowledge is then used to train a binary classifier specifically designed to distinguish between ID and OOD samples. The main challenge is extracting this knowledge from the standard network, especially without access to the original ID training data, and refining it to improve OOD sensitivity. This process requires synthesizing network-specific ID and OOD samples to capture knowledge that is highly sensitive to OOD instances. Furthermore, the extracted knowledge cannot be directly applied to ID-OOD differentiation; it must also be aligned with prior OOD insights to substantially enhance OOD sensitivity.

Inspired by adversarial sample generation~\cite{DBLP:journals/corr/GoodfellowSS14} and diffusion probabilistic models~\cite{DBLP:conf/icml/Sohl-DicksteinW15, DBLP:conf/nips/HoJA20}, a random noise can progress toward an ID sample by incrementally applying subtle perturbations in each transition, enabling the simultaneous synthesis of both ID and OOD samples. These perturbations, anchored in the traditional cross entropy loss~\cite{lin2024diversifying} and a sample constraint, enhance its confidence at each step. The sample constraint guides the synthesis of samples by involving prior knowledge about ID samples in the scenario where training ID samples are unavailable~\cite{mordvintsev2015deepdream, DBLP:conf/cvpr/YinMALMHJK20}. Conversely, when training ID samples are available, it aligns the synthesized samples with their distribution~\cite{DBLP:journals/corr/KingmaW13}. Random noise is viewed as an OOD sample because it follows a distribution different from the training ID. However, an OOD sample with only a few transitions remains OOD, but might exhibit an unexpectedly high-confidence prediction, as illustrated in \cref{fig:conf}. This implies that it is imprudent to fully trust the confidence from the standard network and necessary to encourage the samples in the early stage of the transition to their own low-confidence predictions.

Accordingly, we introduce \textit{Confidence Amendment} (CA) to tackle the challenges associated with OOD knowledge distillation. Based on the observations from the synthesis of ID and OOD samples, CA progressively converts an OOD sample into an ID sample for synthesis, while concurrently enhancing reliance on confidence, thus promoting lower confidence for OOD samples. Accordingly, CA employs a parameterized Markov chain~\cite{DBLP:journals/siammax/DuanWWY20} to convert random noise, treated as OOD sample, into a high-confidence ID sample, synthesizing a sample at each transition. The predicted label distributions from the standard network of these synthesized samples are integrated with a Uniform distribution. Notably, early and later stages of the synthesized samples within this Markov chain carry higher and lower weights on these distributions, respectively. Ultimately, these synthesized samples, with their adjusted predicted label distributions, are utilized to train a binary classifier. This classifier is tailored to discern between samples of high and low confidence levels, thereby equipping it to differentiate between ID and OOD samples.

The main contributions of this paper include:
\begin{itemize}
\item We introduce \textit{Confidence Amendment} (CA), a method that gradually transforms OOD samples into ID-like samples while progressively refining prediction confidence, thereby enhancing network sensitivity to OOD samples.
\item Our approach synthesizes both ID and OOD samples from random noise, facilitating the extraction of OOD-sensitive knowledge from standard networks in both scenarios with and without training ID samples.
\item We establish a generalization error bound showing that refining the knowledge obtained from standard networks substantially improves their ability to distinguish between ID and OOD samples. Extensive experiments validate the effectiveness of our method.\\
\end{itemize}


The rest of this paper is organized as follows: \cref{sec:relatedwork} offers an overview of related techniques and research directions. \cref{sec:algorithm} elaborates on the proposed Confidence Amendment (CA) method. \cref{sec:tg} presents the theoretical guarantees, with \cref{sec:proof} providing the corresponding proofs. \cref{sec:experiment} presents the empirical results. Finally, \cref{sec:conclusion} provides concluding remarks and discusses future directions.

\section{Related Work}\label{sec:relatedwork}
In this section, we introduce OOD detection, knowledge distillation, data-free distillation, and data synthesis.

\subsection{Out-of-distribution Detection}
For a network trained on ID data, OOD detection~\cite{DBLP:journals/tmlr/SalehiMHLRS22, DBLP:journals/corr/abs-2110-11334, DBLP:journals/ijcv/YangZL23} aims to identify samples that deviate from the distribution of the ID ones. Current methods primarily fall into three groups: those that refrain from using training ID data~\cite{DBLP:conf/nips/LeeLLS18, DBLP:conf/icml/HendrycksBMZKMS22, DBLP:conf/nips/SunGL21, DBLP:conf/nips/ZhuCXLZ00ZC22, DBLP:conf/cvpr/OlberRPSC23, DBLP:conf/cvpr/AhnPK23, DBLP:conf/iclr/ZhangF0DLWLH023, DBLP:conf/iclr/GomesADP22, DBLP:conf/icml/ZhuLYLX023}, those that incorporate it~\cite{DBLP:conf/cvpr/0001AB19, DBLP:conf/cvpr/HsuSJK20,DBLP:conf/nips/BibasFH21, DBLP:conf/cvpr/CaoZ22, FIG:23, DBLP:conf/cvpr/0009GLTL022}, and those that incorporate auxiliary OOD data or synthetic outliers during training to improve OOD sensitivity~\cite{DOE:23,DOS:24,GR:24}.

\subsubsection{Methods Not Utilizing Training ID Data}
OOD detection methods that do not use training ID data compute an OOD score based on the outputs of a trained network, without altering the training process or objective. Maximum over Softmax Probability (MSP)~\cite{DBLP:conf/iclr/HendrycksG17} uses the maximum probabilities from softmax distributions to detect OOD samples, as correctly classified examples usually exhibit higher maximum softmax probabilities compared to OOD samples. ODIN~\cite{DBLP:conf/iclr/LiangLS18} improves softmax score-based detection by applying temperature scaling and small input perturbations. Energy-Based Detector (EBD)~\cite{DBLP:conf/nips/LiuWOL20} introduces an energy score for OOD detection, which is more aligned with the probability density of inputs and less prone to overconfidence issues compared to traditional softmax confidence scores. GradNorm~\cite{DBLP:conf/nips/HuangGL21} detects OOD inputs by leveraging information from the gradient space, specifically utilizing the vector norm of gradients derived from the KL divergence between the softmax output and a uniform probability distribution. GEN~\cite{DBLP:conf/cvpr/LiuLZ23} introduces a generalized entropy score function, suitable for any pre-trained softmax-based classifier. Decoupling MaxLogit (DML)~\cite{DBLP:conf/cvpr/ZhangX23} is an advanced logit-based OOD detection method that decouples MaxCosine and MaxNorm from standard logits to enhance OOD detection. ASH~\cite{DBLP:conf/iclr/DjurisicBAL23} is a post-hoc, on-the-fly activation shaping method for OOD detection that removes a significant portion of a late-layer activation during inference without requiring statistics from training data. FeatureNorm~\cite{DBLP:conf/cvpr/YuSLJL23} computes the norm of the feature map from a selected block, rather than the last one, and utilizes jigsaw puzzles as pseudo OOD to select the optimal block. GradGMM and GradPCA~\cite{carvalho2024towards} explore the use of gradient vectors for OOD detection, showing that the geometry of gradients from pre-trained models contains valuable information for distinguishing OOD samples. These methods predominantly hinge on the insights gleaned from trained networks, constraining the potential for elevating OOD sensitivity. On the other hand, OOD knowledge distillation garners OOD-sensitive knowledge by synthesizing samples for a trained network, unveiling its distribution vulnerabilities and bolstering its sensitivity to OOD.

\subsubsection{Methods Utilizing Training ID Data}
OOD detection methods that use training ID data improve the OOD sensitivity of a trained network by either maintaining or fine-tuning it with ID data while incorporating OOD prior knowledge. Confidence-Calibrated Classifier (CCC)~\cite{DBLP:conf/iclr/LeeLLS18} adds two terms to the cross-entropy loss: one to reduce confidence in OOD samples and another for generating beneficial training samples, jointly training classification and generative networks. Minimum Others Score (MOS)~\cite{huang2021mos} refines the semantic space by grouping similar concepts, simplifying decision boundaries for OOD detection. G-ODIN~\cite{DBLP:conf/cvpr/HsuSJK20} improves ODIN by using decomposed confidence scoring and modified input preprocessing. Density-Driven Regularization (DDR)~\cite{DBLP:conf/nips/HuangWXW022} enforces density consistency and contrastive distribution regularization to separate ID and OOD samples. Watermarking~\cite{DBLP:conf/nips/WangLZZ0L022} utilizes the reprogramming capabilities of deep models to enhance OOD detection via distinct feature perturbations without altering model parameters. Adversarial Reciprocal Point Learning (ARPL)~\cite{DBLP:journals/pami/ChenPWT22} introduces reciprocal points and adversarial constraints to reduce class overlap. ViM~\cite{wang2022vim} integrates features and logits to produce a softmax score for a virtual OOD class. CIDER~\cite{DBLP:conf/iclr/MingSD023} uses hyperspherical embeddings with dispersion and compactness losses for robust prototype separation. HEAT~\cite{DBLP:conf/icml/LafonRRT23} provides an energy-based solution to address MCMC sampling issues in energy-based models. While these methods advance OOD detection, they often require extra training phases for fine-tuning or retraining, which can affect generalization and scalability. In contrast, OOD knowledge distillation derives OOD-sensitive binary classifiers from a trained network without altering the network itself.

\subsubsection{Methods Utilizing Training OOD Data}
Several recent methods have shown improved OOD detection performance by explicitly leveraging auxiliary OOD datasets during training. For example, Diversified Outlier Exposure (DOE)~\cite{DOE:23} employs informative extrapolation guided by diversity-aware objectives to synthesize effective OOD samples for training. Similarly, Diverse Outlier Sampling (DOS)~\cite{DOS:24} proposes a diverse outlier sampling strategy by applying K-Means clustering in feature space to select representative OOD samples, improving the informativeness and diversity of training data. GReg~\cite{GR:24} further introduces gradient-based regularization to enforce local consistency in the score function and pairs it with an energy-based sampling algorithm, leading to stronger OOD robustness. These approaches demonstrate the benefit of exposing models to carefully selected or synthesized OOD examples, in contrast to purely post-hoc scoring-based detection. Our work differs by not assuming access to diverse or well-structured OOD training data, and instead focuses on scoring-based separation without requiring external outlier exposure.

\begin{figure*}[t]
\centering
\includegraphics[width=1\textwidth]{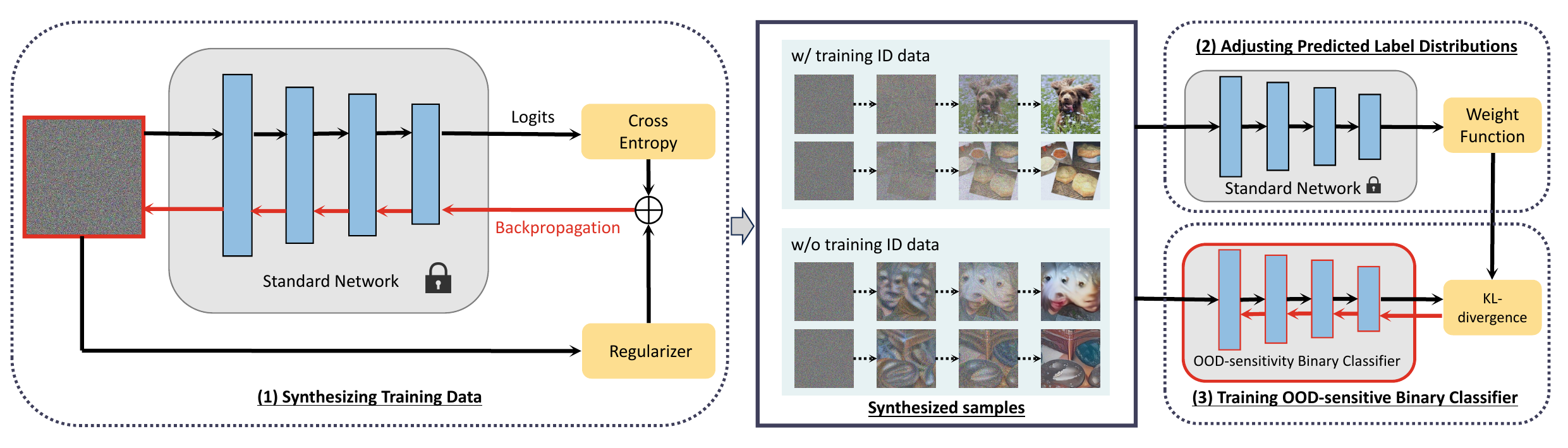}
\caption{Framework of Confidence Amendment (CA). It enhances OOD detection by progressively distilling OOD-sensitive knowledge from a pre-trained standard network to train a binary classifier that distinguishes between ID and OOD samples. The framework consists of three main sequential components: (1) Synthesizing Training Data through a Markov chain to gradually transform OOD samples into ID samples (\cref{sec:CR1}), (2) Adjusting Predicted Label Distributions by applying a weight function to the predicted label distributions of synthesized samples, ensuring low-confidence predictions for OOD samples (\cref{sec:CR2}), and (3) Training an OOD-sensitive Binary Classifier using these adjusted predictions and knowledge distillation, enabling precise ID and OOD differentiation based on confidence levels (\cref{sec:CR3}). The red line represents backpropagation to update the parameters.}
\label{fig:fw}
\end{figure*}

\subsection{Knowledge and Data-free Distillation}
Knowledge distillation~\cite{DBLP:conf/iclr/Allen-ZhuL23, DBLP:journals/ijcv/GouYMT21} involves training a smaller model to replicate the behavior of a larger, more complex model. A foundational approach leverages the soft outputs of the teacher model to train the student~\cite{DBLP:journals/corr/HintonVD15}, while FitNets~\cite{DBLP:journals/corr/RomeroBKCGB14} use intermediate representations for guidance. Generalized distillation extends this concept, linking it to privileged information~\cite{DBLP:journals/corr/Lopez-PazBSV15}. Traditional knowledge distillation relies on access to the original training data. Data-free distillation addresses this by synthesizing data to match the feature statistics~\cite{DBLP:conf/iccv/ChenW0YLSXX019} of the teacher model or by generating data resembling the original training set through transformations~\cite{DBLP:journals/corr/abs-1710-07535} and iterative refinement~\cite{DBLP:conf/aaai/HeoLY019}.

While traditional and data-free distillation focus on transferring ID classification knowledge, our proposed OOD knowledge distillation framework transfers OOD-sensitive knowledge, specifically enabling the student model to distinguish between ID and OOD samples. One related work is multi-level knowledge distillation~\cite{wu-etal-2023-multi}, which utilizes a layered distillation approach designed around the semantic structure of text, relying heavily on language model features and self-supervised training on ID data. However, this approach solely focuses on knowledge distillation without enhancing OOD sensitivity. Its reliance on text-specific semantic representations and extensive ID data further limits its applicability to text-based OOD detection tasks, reducing flexibility for other data types or scenarios with limited ID samples. In contrast, the proposed method introduces a dynamic confidence adjustment mechanism that progressively transforms OOD samples into ID samples through a parameterized Markov chain. CA synthesizes both ID and OOD samples directly, enabling OOD detection without requiring extensive ID data and preserving the original network structure.

\subsection{Data Synthesis}
A variety of methods have been proposed to synthesize artificial data that closely mirrors real-world data. Variational Autoencoder (VAE)~\cite{DBLP:journals/corr/KingmaW13} is a generative model that learns to encode input data into a latent space and then decodes to produce new data samples that mirror the input distribution. Generative Adversarial Networks (GAN)~\cite{DBLP:conf/nips/GoodfellowPMXWOCB14} employs a dual network structure, where a generator crafts synthetic data while a discriminator assesses its authenticity, collaboratively refining the generation process. DeepDream~\cite{mordvintsev2015deepdream} iteratively modifies images to enhance the patterns recognized by a neural network, leading to dream-like generated images. DeepInversion~\cite{DBLP:conf/cvpr/YinMALMHJK20} inverts the roles in the training process, aiming to generate images that maximize the response of particular neurons, providing insights into what deep networks perceive. Existing data synthesis methods primarily focus on the final generated samples. However, our algorithm emphasizes the entire generation process where samples gradually transition from OOD to ID, with their confidence levels steadily increasing. Every sample produced throughout this process is fully utilized by our approach.

\section{Confidence Amendment}\label{sec:algorithm}
OOD knowledge distillation extracts information sensitive to OOD samples from a standard network to train its specialized binary classifier, tailored to discriminate between ID and OOD samples. Let \( \mathbf{x} \in \mathcal{X} \) represent the input and \( y \in [K] \) the associated label, with \( K \) is the total number of labels. The standard network, denoted as \( \mathcal{P}_{\theta}(y | \mathbf{x}) \) and parameterized by \( \theta \), is trained using an ID dataset \( \mathbf{O} = \{(\mathbf{x}_i,y_i)\}_{i = 1}^N \). This dataset consists of \( N \) independent and identically distributed samples drawn from an unknown distribution. The information sensitive to OOD samples is extracted from \( \mathcal{P}_{\theta}(y | \mathbf{x}) \) to train a binary classifier \( \mathcal{P}_{\phi}(c | \mathbf{x}) \), parameterized by \( \phi \), where \( c \in \{0,1\} \) to distinguish between ID and OOD samples. Here, \( c = 1 \) signifies that the test sample \( \mathbf{x} \) is ID, while \( c = 0 \) indicates an OOD sample. In the testing phase, the classifier \( \mathcal{P}_{\phi}(c | \mathbf{x}) \) determines whether a given input \( \mathbf{x} \) is ID or OOD. If identified as ID, the standard network \( \mathcal{P}_{\theta}(y | \mathbf{x}) \) is utilized to predict its label. Conversely, if it is determined to be OOD, the standard network abstains from making a label prediction.

Confidence Amendment (CA), visualized in \cref{fig:fw}, is designed to address the challenges associated with OOD knowledge distillation, specifically those involving the extraction of knowledge from a standard network and the subsequent refinement of this knowledge for training an OOD-sensitive binary classifier. Specifically, for the given standard network \( \mathcal{P}_{\theta}(y | \mathbf{x}) \), CA procedure begins by synthesizing a dataset \( \mathbf{S} = \{\widehat{\mathbf{x}}_i\}_{i = 1}^M \) containing both ID and OOD samples and extracting their corresponding predicted label distributions from the standard network. Drawing on the foundational understanding that OOD samples are anticipated to align with a uniform distribution, the predicted label distributions are melded with a uniform distribution using adaptive weights, thereby augmenting the sensitivity towards OOD samples. The synthesized samples with adjusted predicted label distributions are applied for training the binary classifier \( \mathcal{P}_{\phi}(c | \mathbf{x}) \), which is tasked with distinguishing between ID and OOD samples.

\subsection{Synthesizing Training Data}\label{sec:CR1}
CA synthesizes samples, including both ID and OOD samples, through a parameterized Markov chain. In the transitions within this chain, an OOD sample is gradually converted into an ID sample by elevating its confidence with respect to the standard network \( \mathcal{P}_{\theta}(y | \mathbf{x}) \). Consequently, specific ID and OOD samples for the standard network can be effectively synthesized.

Accordingly, a random noise \( \widehat{\mathbf{x}}_0 \) drawn from a standard distribution \( \mathcal{N}(\mathbf{0}, \mathbf{I}) \) can be considered an OOD sample, as the training ID samples from \( \mathbf{O} \) do not follow this standard distribution, i.e.,
\begin{equation}
\label{eq:x0}
\widehat{\mathbf{x}}_0 \sim \mathcal{N}(\mathbf{0}, \eta \mathbf{I}).
\end{equation}
Taking inspiration from diffusion probabilistic models~\cite{DBLP:conf/icml/Sohl-DicksteinW15, DBLP:conf/nips/HoJA20}, the randomly-initialized OOD sample incrementally transitions to an ID sample after \( T \) transformations within a Markov chain, defined as follows:
\begin{equation}
\label{eq:chain}
\mathcal{P}_t(\widehat{\mathbf{x}}_{t} | \widehat{\mathbf{x}}_{t - 1})  = \mathcal{N}(\bm{\mu}_{t}, \eta \mathbf{I}), \quad t \in [1, T],
\end{equation}
where \( \eta \) is the variance, \( \bm{\mu}_{t} \) is the expectation of \( \widehat{\mathbf{x}}_t \), and \( T \) represents the maximum transition time. To facilitate the evolution of an OOD sample into an ID sample for the standard network, the confidence level of the sample needs enhancement, as ID samples typically exhibit high-confidence predictions. Drawing from the principles of adversarial sample generation models~\cite{DBLP:journals/corr/GoodfellowSS14}, the confidence of a randomly-initialized OOD sample can be boosted by introducing a small, informative perturbation that relates to both the standard network and a random label. Consequently, the expectation \( \bm{\mu}_{t} \) can be expressed as:
\begin{equation}
\label{eq:per}
\bm{\mu}_{t} = \widehat{\mathbf{x}}_{t - 1} - \rho \nabla G_{\theta}(\widehat{\mathbf{x}}_{t - 1}).
\end{equation}
Here, \( \rho \) is a coefficient denoting the magnitude of the perturbation. The term \( G_{\theta}(\widehat{\mathbf{x}}) \) is based on the standard network \( \mathcal{P}_{\theta}(y | \widehat{\mathbf{x}}) \) and incorporates a regularizer \( \mathcal{R}(\widehat{\mathbf{x}}) \) applied to the synthesized sample \( \widehat{\mathbf{x}} \), formulated as:
\begin{equation}
G_{\theta}(\widehat{\mathbf{x}}) = - \log \mathcal{P}_{\theta}(y | \widehat{\mathbf{x}}) + \mathcal{R}(\widehat{\mathbf{x}}),
\label{eq:G}
\end{equation}
where \( \mathcal{R}(\widehat{\mathbf{x}}) \) encourages the distribution of synthesized samples to closely align with that of the original training samples.

Specifically, when the training dataset \( \mathbf{O} \) is available, inspired by the concept of the variational autoencoder~\cite{DBLP:journals/corr/KingmaW13}, the distribution discrepancy between the real and synthesized samples can be minimized. Accordingly, the regularizer \( \mathcal{R}(\widehat{\mathbf{x}}) \) applicable when the training datasets are available, which is termed DeepRecon, can be defined as follows:
\begin{equation}
\mathcal{R}(\widehat{\mathbf{x}}) = \mathcal{R}^+(\widehat{\mathbf{x}}) = \beta_{\text{MSE}}\text{MSE}(\widehat{\mathbf{x}}, \mathbf{x}),
\label{eq:RS}
\end{equation}
where MSE represents the mean squared error scaled by factor \( \beta_{\text{MSE}} \), and \( \mathbf{x} \) is a sample randomly selected from the training dataset \( \mathbf{O} \). Additionally, the label used in the standard network \( \mathcal{P}_{\theta}(y | \widehat{\mathbf{x}}) \) in the computation of \( G_{\theta}(\widehat{\mathbf{x}}) \) is the ground-truth label \( y \) corresponding to the randomly-selected \( \mathbf{x} \). Alternatively, when the training dataset \( \mathbf{O} \) is unavailable, one can regularize the distribution of synthesized samples by using priors, a strategy inspired by DeepDream~\cite{mordvintsev2015deepdream} and DeepInversion~\cite{DBLP:conf/cvpr/YinMALMHJK20}, which ensures stable convergence towards valid samples. In this case, the regularizer \( \mathcal{R}(\widehat{\mathbf{x}}) \) for unavailable training datasets can be expressed as:
\begin{equation}
\mathcal{R}(\widehat{\mathbf{x}}) = \mathcal{R}^-(\widehat{\mathbf{x}}) = \beta_{\text{TV}} \mathcal{R}_{\text{TV}}(\widehat{\mathbf{x}}) + \beta_{l_2} \mathcal{R}_{l_2}(\widehat{\mathbf{x}}) + \beta_{\text{f}} \mathcal{R}_{\text{f}}(\widehat{\mathbf{x}}),
\label{eq:RU}
\end{equation}
where \( \mathcal{R}_{\text{TV}}(\widehat{\mathbf{x}}) \), \( \mathcal{R}_{l_2}(\widehat{\mathbf{x}}) \), and \( \mathcal{R}_{\text{f}}(\widehat{\mathbf{x}}) \) penalize the total variance, \( l_2 \) norm, and the distribution of intermediate feature maps of \( \widehat{\mathbf{x}} \), respectively, each scaled by their corresponding factors \( \beta_{\text{TV}} \), \( \beta_{l_2} \), and \( \beta_{\text{f}} \). The three regularization terms are introduced in DeepInversion~\cite{DBLP:conf/cvpr/YinMALMHJK20}.

As per \cref{eq:chain}, \cref{eq:per}, and \cref{eq:G}, coupled with the application of the reparameterization trick~\cite{DBLP:conf/icml/GalG16}, the synthesized sample \( \widehat{\mathbf{x}}_{t} \) at time \( t \in [1, T] \) within the parameterized Markov chain can be determined in closed form as follows:
\begin{equation}
\widehat{\mathbf{x}}_{t} = \widehat{\mathbf{x}}_{t - 1} + \rho \nabla \log \mathcal{P}_{\theta}(y | \widehat{\mathbf{x}}_{t - 1}) - \rho \nabla \mathcal{R}(\widehat{\mathbf{x}}_{t - 1}) + \eta \mathbf{z},
\label{eq:xt}
\end{equation}
where \( \mathbf{z} \) is a random variable following a standard distribution, i.e., \( \mathbf{z} \sim \mathcal{N}(\mathbf{0}, \mathbf{I}) \). Given a random initial dataset \( \mathbf{S}_0 = \{\widehat{\mathbf{x}}_{i,0}\}_{i = 1}^N \) consisting of \( N \) independent random variables drawn from \( \mathcal{N}(\mathbf{0}, \mathbf{I}) \), the corresponding synthesized data subset \( \mathbf{S}_t = \{\widehat{\mathbf{x}}_{i,t}\}_{i = 1}^N \) at time \( t \in [1,T] \) can be derived through \cref{eq:xt}. Consequently, by aggregating all such datasets across the various time steps, we obtain the synthesized dataset:
\begin{equation}
\mathbf{S} =  \bigcup_{t = 0}^T \mathbf{S}_t = \{\widehat{\mathbf{x}}_{i,0:T}\}_{i = 1}^N,
\end{equation}
which encompasses \( N(T + 1) \) samples. Notably, the OOD samples present at time \( 0 \) gradually evolve into ID samples as time progresses to \( T \). Therefore, the dataset \( \mathbf{S} \) encapsulates both distinct OOD and ID samples pertinent to the standard network \( \mathcal{P}_{\theta}(y | \mathbf{x}) \).

\subsection{Adjusting Predicted Label Distributions}\label{sec:CR2}
For samples originating from the synthesized dataset \( \mathbf{S} \), their predicted label distributions can be retrieved from the standard network \( \mathcal{P}_{\theta}(y | \mathbf{x}) \). Samples in \( \mathbf{S} \) with few transitions can be regarded as OOD, owing to the substantial discrepancy between their distribution and that of the ID. However, as illustrated in \cref{fig:conf}, these samples might receive unexpectedly high-confidence predictions from the standard network, despite their characteristics. This phenomenon arises due to the distributional vulnerability of the standard network~\cite{FIG:23}. While the network is trained on ID samples, it does not have constraints imposed on OOD samples. This can lead to uncertain and occasionally high-confidence predictions for OOD samples. Thus, utilizing the synthesized samples and their predicted label distributions from the standard network directly for training a binary classifier would not enhance the OOD sensitivity of the network.

To improve OOD sensitivity, refining the extracted knowledge by adjusting the predicted label distributions of synthesized samples is necessary, ensuring that OOD samples correlate with low-confidence predictions. The fundamental idea behind this approach is to incrementally place trust in the prediction confidence. Specifically, in the process of synthesizing samples, earlier samples are more likely to be OOD, and therefore their high-confidence predictions are not reliable. In contrast, later samples tend to be ID, and their high-confidence predictions are reliable. Therefore, for a synthesized sample \( \widehat{\mathbf{x}}_{i,t} \) at time \( t \), with \( i \in [N] \) and \( t \in [0,T] \), the adjusted predicted label distribution can be computed as:
\begin{equation}
\mathcal{Q}_{\theta}(y | \widehat{\mathbf{x}}_{i,t}) = \left( 1 - \alpha(t) \right) \mathcal{U} + \alpha(t) \mathcal{P}_{\theta}(y | \widehat{\mathbf{x}}_{i,t}),
\label{eq:Q}
\end{equation}
where \( \mathcal{U} \) denotes the uniform distribution and \( \alpha \) represents a weight function defined as:
\begin{equation}
\alpha(t) = \left(\frac{t}{T} \right)^a, \quad a \in [0, +\infty).
\end{equation}
\cref{fig:alpha} displays the curves of the weight function for various coefficients of \( a \geq 0 \). When \( a = 0 \), all function values are unity, suggesting complete reliance of the synthesized samples on the confidence provided by the standard network. For \( a > 0 \), the function exhibits a monotonic increase, indicating that the synthesized samples will progressively trust the confidence levels from the standard network, with higher trust accorded as a sample approaches ID characteristics. Consequently, with \( a > 0 \), confidence levels from the standard network are revised in the process of confidence amendment, assigning lower confidence to OOD samples and higher confidence to ID samples, thereby heightening OOD sensitivity.

This confidence-based supervision in~\cref{eq:Q} brings several key benefits. First, confidence scores are readily available from most probabilistic models and require no additional computation or architectural modification, making our method broadly applicable across different backbones. Second, confidence offers strong interpretability, as it directly reflects the belief of the model in its prediction, which is particularly important for understanding and mitigating overconfidence on OOD samples. Third, our formulation introduces a flexible interpolation mechanism via $\alpha(t)$, allowing smooth adjustment of trust in the prediction of the model during the synthesis process. Finally, this simple yet effective structure facilitates theoretical analysis, as shown in our generalization bound in~\cref{sec:tg}, and allows for straightforward tuning in practice.

\begin{figure}
  \centering
  \includegraphics[width=0.48\textwidth]{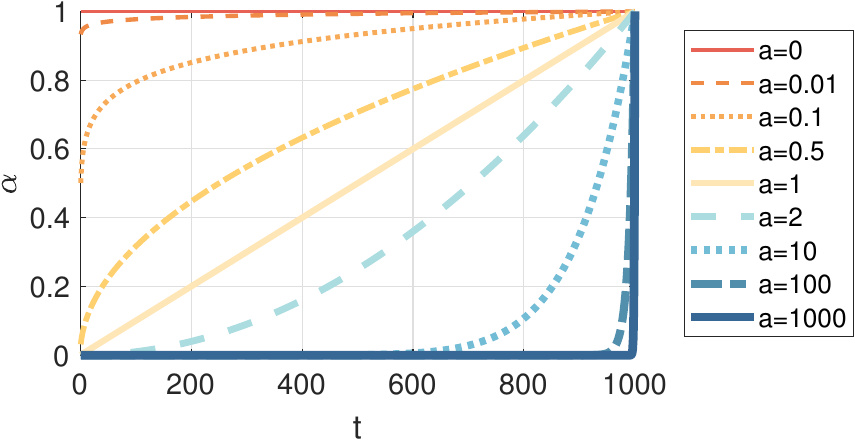}\\
  \caption{Curves of the function $\alpha(t)$ under different parameters. Best viewed in color.}
  \label{fig:alpha}
\end{figure}

\subsection{Training OOD-sensitive Binary Classifier}\label{sec:CR3}
To harness the deeper knowledge encapsulated within the standard network \( \mathcal{P}_{\theta}(y | \mathbf{x}) \), for a given input \( \mathbf{x} \), we aim to map it to the adjusted predicted label distribution \( \mathcal{Q}_{\theta}(y | \mathbf{x}) \) using an auxiliary network \( \mathcal{P}_{\phi}(y | \mathbf{x}) \) parameterized by \( \phi \). Subsequently, a specialized binary classifier capable of distinguishing between ID and OOD samples can be devised based on this auxiliary network. Following conventional knowledge distillation approaches, OOD-sensitive information from \( \mathcal{Q}_{\theta}(y | \mathbf{x}) \) can be transferred to \( \mathcal{P}_{\phi}(y | \mathbf{x}) \) by optimizing the objective
\begin{equation}
\min_{\phi} \sum_{\mathbf{x} \in \mathbf{S}}  \sum_{y \in [K]} \mathcal{D}_{\text{KL}} \left( \mathcal{P}_{\phi}(y | \mathbf{x}) || \mathcal{Q}_{\theta}(y | \mathbf{x}) \right),
\end{equation}
where \( \mathcal{D}_{\text{KL}} \left( \cdot || \cdot \right) \) denotes the Kullback-Leibler (KL) divergence. Importantly, the classifier is not trained on raw predictions from the original network. Instead, we use the amended confidence distribution $\mathcal{Q}_{\theta}(y | \mathbf{x})$, which dynamically interpolates between the network prediction and a uniform prior. This ensures that the supervision signal is calibrated according to the synthesis stage of the input, helping the classifier avoid overfitting to unreliable high-confidence outputs from the base model on OOD samples. As a result, the binary classifier focuses on uncertainty-aware decision boundaries that are more robust and generalizable.

Inspired by the maximum over softmax probability technique~\cite{DBLP:conf/iclr/HendrycksG17}, which computes an OOD score for a test sample based on prediction confidence, we can formulate the specialized binary classifier for the standard network \( \mathcal{P}_{\theta}(y | \mathbf{x}) \) using the auxiliary network \( \mathcal{P}_{\phi}(y | \mathbf{x}) \):
\begin{equation}
\begin{aligned}
\mathcal{P}_{\phi}(c = 1| \mathbf{x}) & = \max_{y \in [K]} \mathcal{P}_{\phi}(y | \mathbf{x}),\\
\mathcal{P}_{\phi}(c = 0| \mathbf{x}) & = 1 - \max_{y \in [K]} \mathcal{P}_{\phi}(y | \mathbf{x}),
\end{aligned}
\end{equation}
with \( c = 1 \) signifying that the test sample \( \mathbf{x} \) is ID and \( c = 0 \) denoting an OOD sample. Thus, \( \mathcal{P}_{\phi}(c| \mathbf{x}) \) acts as the specialized binary classifier corresponding to the standard network \( \mathcal{P}_{\theta}(y | \mathbf{x}) \). This classifier, trained with specific samples derived from the standard network, is tailored to differentiate between ID and OOD, exhibiting sensitivity to the latter. During testing, for a given sample \( \mathbf{x} \), the value of $\mathcal{P}_{\phi}(c = 0| \mathbf{x})$ serves as the OOD score. A higher score suggests a greater likelihood that the sample is OOD. The process of the proposed CA is outlined in \cref{alg:CA}.

\begin{algorithm}[t]
    \caption{Confidence Amendment (CA)}
    \label{alg:CA}
    \begin{algorithmic}[1]
    \Require Standard network $\mathcal{P}_{\theta}(y | \mathbf{x})$, weight function coefficient $a$, maximum transition time $T$
    \State Synthesize a dataset $\mathbf{S}$ by integrating samples at $t \in [0,T]$ in the Markov chain:
    \begin{equation*}
    \widehat{\mathbf{x}}_{t} = \widehat{\mathbf{x}}_{t - 1} + \rho \nabla \log \mathcal{P}_{\theta}(y | \widehat{\mathbf{x}}_{t - 1}) - \rho \nabla \mathcal{R}(\widehat{\mathbf{x}}_{t - 1}) + \eta \mathbf{z}.
    \end{equation*}
    \State For each synthesized sample $\widehat{\mathbf{x}}_{i,t} (i \in [N], t \in [0,T])$ in $\mathbf{S}$, adjust the predicted label distribution:
    \begin{equation*}
    \mathcal{Q}_{\theta}(y | \widehat{\mathbf{x}}_{i,t}) = \alpha(t) \mathcal{U} + \left(1 - \alpha(t)\right) \mathcal{P}_{\theta}(y | \widehat{\mathbf{x}}_{i,t}).
    \end{equation*}
    \State Distill knowledge from $\mathcal{Q}_{\theta}(y | \mathbf{x})$ into auxiliary network $\mathcal{P}_{\phi}(y| \mathbf{x})$ by optimizing:
    \begin{equation*}
    \min_{\phi} \sum_{\mathbf{x} \in \mathbf{S}}  \sum_{y \in [K]} \mathcal{D}_{\text{KL}} \left( \mathcal{P}_{\phi}(y | \mathbf{x}) || \mathcal{Q}_{\theta}(y | \mathbf{x}) \right).
    \end{equation*}
    \State Formulate specialized binary classifier $\mathcal{P}_{\phi}(c| \mathbf{x})$ based on auxiliary network $\mathcal{P}_{\phi}(y| \mathbf{x})$:
    \begin{equation*}
    \begin{aligned}
    \mathcal{P}_{\phi}(c = 1| \mathbf{x}) & = \max_{y \in [K]} \mathcal{P}_{\phi}(y | \mathbf{x}),  \\
    \mathcal{P}_{\phi}(c = 0| \mathbf{x}) & = 1 - \max_{y \in [K]} \mathcal{P}_{\phi}(y | \mathbf{x}).
    \end{aligned}
    \end{equation*}
    \Ensure OOD-sensitive binary classifier $\mathcal{P}_{\phi}(c| \mathbf{x})$
    \end{algorithmic}
\end{algorithm}

\section{Theoretical Guarantees}\label{sec:tg}
In this section, we present \cref{co:a}, which establishes a generalization error bound for our OOD-sensitive binary classifier. The core of our analysis lies in understanding how the weight function \(\alpha(t)\), parameterized by the weighting coefficient \(a\), affects the ability of the specialized binary classifier $\mathcal{P}_{\phi}(c| \mathbf{x})$ to distinguish between ID and OOD samples. By deriving this bound, we aim to guide the selection of \(a\) to optimize the sensitivity of the classifier to OOD samples, thereby demonstrating the effectiveness of our algorithm.

This analysis is grounded in fundamental concepts of fat-shattering dimensions $\text{fat}(\cdot)$~\cite{EDFL:94} and covering numbers~\cite{ML:14}, which provide the structural basis for Vapnik's method of structural risk minimization~\cite{SRM:98}. For convenience, we assume the hypothesis space of the specialized binary classifiers $\mathcal{P}_{\phi}(c| \mathbf{x})$ is denoted as $\mathcal{H}$, then we have
\begin{equation}
h(\mathbf{x}) = \mathcal{P}_{\phi}(c = 0| \mathbf{x}).
\end{equation}
Let $\mathcal{P}_\mathbf{S}$ represent the mixture distribution of ID and OOD samples drawn from $\mathbf{S}$, and define $l(h(\mathbf{x}), c) = \mathbf{I}[h(\mathbf{x}) = c]$ as the $0$-$1$ loss function. Then, the expected and empirical risks of $h(\mathbf{x})$ can be expressed as:
\begin{equation*}
\mathcal{L}_{\mathcal{P}_\mathbf{S}} \left[ h \right] = \int_{\mathcal{P}_\mathbf{S}} l(h(\mathbf{x}),c) \,d \mathbf{x}, \quad \mathcal{L}_{\mathbf{S}} \left[ h \right] = \frac{1}{\vert \mathbf{S}\vert} \sum_{\mathbf{x} \sim \mathbf{S}} l(h(\mathbf{x}),c).
\end{equation*}
\begin{theorem}\label{co:a}
Consider a hypothesis space \(\mathcal{H}\) confined within a ball of radius \(R\). Let \(h \in \mathcal{H}\) be a hypothesis that correctly classifies \(M = N(T + 1)\) samples from \(\mathbf{S} \in \mathcal{P}\) with a margin \(\gamma_t = \zeta - r_t\) and fat-shattering dimension \(\kappa_t = \text{fat}(\gamma_t / 8)\) for each dataset \(\mathbf{S}_t\) assigned to \(K\) classes. Here, \(\zeta \geq 1\), \(r_t = \max_{\mathbf{x} \in \mathbf{S}_t} h(\mathbf{x})\), and \(t \in [0, T]\). Let \(\alpha(t) = \left(\frac{t}{T}\right)^a\) for \( t \in [0,T] \) be a weight function used to smooth the output distribution of a standard network, where \( a \geq 0 \). With probability at least \( 1 - \delta \), the approximate generalization error bound is given by
\begin{equation*}
\mathcal{L}_{\mathcal{P}_{\mathcal{S}}}[h] \leq \frac{ 620 R^2 \log \left( 32M \right) K^2 \varphi(a)}{4T\sqrt{M^3} (K - 1)^2 } + \frac{9}{\sqrt{N \delta}},
\end{equation*}
where
\begin{equation*}
\varphi(a) = \frac{(a + 1)(2a + 1)}{a^2},
\end{equation*}
which depends on the weight function coefficient \(a \geq 0\) and influences the generalization error bound.
\end{theorem}

\cref{co:a} demonstrates that the weighting coefficient \( a \geq 0 \) in the weight function \(\alpha(t)\) influences the generalization error bound of the specialized binary classifier through the function \(\varphi(a)\). Since
\begin{equation}
\frac{d \varphi(a)}{da} = \frac{-3a - 2}{a^3} < 0, \forall a > 0,
\label{eq:da}
\end{equation}
\(\varphi(a)\) is monotonically decreasing with respect to \( a \). Thus, by utilizing \( M = N(T + 1) \) synthesized samples from the parameterized Markov chain in \cref{eq:xt} and applying \(\alpha_t = \left(\frac{t}{T}\right)^a\) to integrate knowledge smoothly from the standard network, the binary classifier designed to distinguish between ID and OOD samples can achieve a lower generalization error bound with a larger weighting coefficient \( a \geq 0 \). The function curve of \(\varphi(a)\) is illustrated in \cref{fig:curve}, providing insight into how adjusting \(a\) impacts the performance of the classifier. This generalization error bound helps in selecting an optimal \(a\) to improve OOD sensitivity, thereby enhancing the robustness and reliability of the classifier in distinguishing between ID and OOD samples.

This theoretical result validates our core insight regarding the importance of adjusting confidence levels to improve OOD sensitivity. As discussed in \cref{sec:CR2}, simply using the predicted label distributions of synthesized samples directly from the standard network without adjustment (\(a = 0\)) is insufficient, as OOD samples may receive high-confidence predictions due to the distributional limitations of the network. By tuning \(a\) in the weighting function \(\alpha(t)\), we can progressively lower the confidence for OOD samples while increasing it for ID samples as they transition through the Markov chain. This mechanism effectively aligns low-confidence predictions with OOD samples and high-confidence predictions with ID samples, thereby enhancing the OOD sensitivity of the classifier.

\begin{figure}
  \centering
  \includegraphics[width=0.4\textwidth]{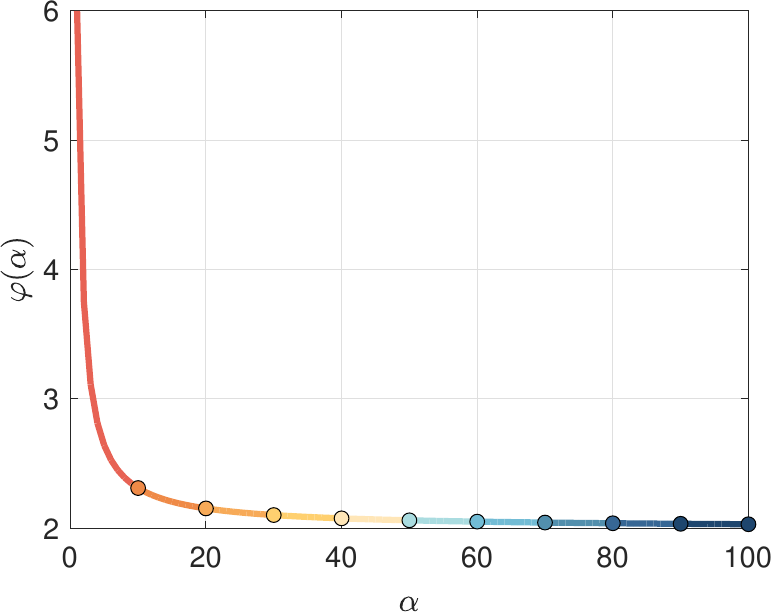}
  \caption{Function Curve of $\varphi(a)$ with respect to $a \geq 0$.}
  \label{fig:curve}
\end{figure}

\section{Proofs}\label{sec:proof}
This section outlines the necessary definitions, assumptions, and lemmas that support the derivation of the bound in \cref{co:a}, along with the proof process, leading to a clear understanding of the interplay between the weight function and the generalization performance of the OOD-sensitive classifier.

\subsection{Definitions and Preliminary Lemmas}
Before deriving the generalization error bound, we introduce essential definitions and lemmas that form the theoretical foundation for the analysis.
\begin{definition}[Fat Shattering Dimension~\cite{EDFL:94}]\label{def:func}
Let $\mathcal{H}$ be a set of real-valued functions. A set of points $\mathcal{X}$ is said to be $\gamma$-shattered by $\mathcal{H}$ if there exist real numbers $r_\mathbf{x}$, each indexed by $\mathbf{x} \in \mathcal{X}$, such that for all binary vectors $b$, also indexed by $\mathbf{x}$, there exists a function $h_b \in \mathcal{H}$ satisfying
\begin{equation*}
h_b(\mathbf{x})  =\left \{
\begin{array}{ll}
  r_\mathbf{x} + \gamma & \text{if} \quad b_\mathbf{x} = 1,\\
  r_\mathbf{x} - \gamma & \text{otherwise.} \\
\end{array} \right.
\end{equation*}
The fat-shattering dimension denoted as $\text{fat}_{\mathcal{H}}$, of the set $\mathcal{H}$ is a function mapping positive real numbers to integers. Specifically, it assigns a value $\gamma$ to the size of the largest set $\mathcal{X}$ that is $\gamma$-shattered by $\mathcal{H}$, yielding infinity if no such finite set exists.
\end{definition}
\begin{definition}[$\epsilon$-covering~\cite{ML:14}]\label{def:cover}
Let $(\mathcal{X}, d)$ be a (pseudo-)metric space and $\mathcal{A}$ a subset of $\mathcal{X}$ with a specified $\epsilon > 0$. A set $\mathcal{B} \subseteq \mathcal{A}$ is called an $\epsilon$-cover for $\mathcal{A}$ if, for every element $A \in \mathcal{A}$, there exists an element $B \in \mathcal{B}$ satisfying $d(A, B) \leq \epsilon$. The $\epsilon$-covering number of $\mathcal{A}$, denoted as $N(\epsilon, \mathcal{A})$, represents the minimal cardinality of an $\epsilon$-cover for $\mathcal{A}$. This number is defined to be infinite ($\infty$) if no finite $\epsilon$-cover exists for $\mathcal{A}$.
\end{definition}
\begin{lemma}[Covering Number~\cite{SRM:98}]\label{lm:cover}
Let $\mathcal{H}$ be a class of functions mapping $\mathcal{X} \rightarrow [b_1,b_2]$, and let $\mathcal{P}_\mathcal{X}$ represent a distribution over $\mathcal{X}$. Given \(0 < \epsilon < 1\), set \(\kappa = \text{fat}_{\mathcal{H}}(\epsilon / 4)\). Then, the expectation
\begin{equation*}
\mathbb{E}_{\mathcal{X}^m} \left( \mathfrak{N}(\epsilon, \mathcal{H})\right) \leq 2 \left( \frac{4m (b_2 - b_1)^2}{\epsilon^2} \right)^{\kappa \log\frac{2em(b_2 - b_1)}{\kappa \epsilon}}
\end{equation*}
is taken over $m$ samples $\mathbf{S} \in \mathcal{X}^m$ drawn in accordance with the distribution $\mathcal{P}_\mathcal{X}$.
\end{lemma}
\begin{lemma}[Symmetrization~\cite{ED:06}]\label{lm:sm}
Let $\mathcal{H}$ be a class of real-valued functions and let $\mathcal{P}_\mathcal{X}$ be a probability measure on $\mathcal{X}$. Let $\mathbf{S}$ and $\bar{\mathbf{S}}$ each contain $m$ samples, both drawn independently according to $\mathcal{P}_\mathcal{X}$. If $m > 2 / \epsilon$, then we have
\begin{equation*}
\begin{aligned}
& \mathcal{P}_{\mathbf{S}} \left( \sup_{h \in \mathcal{H}} \vert \mathcal{L}_{\mathbf{S}}[h] - \mathcal{L}_{\mathcal{P}_{\mathcal{S}}}[h]\vert \geq \varepsilon\right) \\
\leq & \mathcal{P}_{\mathbf{S}\bar{\mathbf{S}}} \left( \sup_{h \in \mathcal{H}} \vert \mathcal{L}_{\mathbf{S}}[h] - \mathcal{L}_{\bar{\mathbf{S}}}[h]\vert \geq \varepsilon / 2\right).
\end{aligned}
\end{equation*}
\end{lemma}
\begin{lemma}[Fat Shattering Dimension Bound~\cite{GZ:99}]\label{lm:fb}
Suppose $\mathcal{H}$ is confined to points within an $n$-dimensional ball of radius $R$ centered at the origin. Then, we have
\begin{equation*}
\text{fat}_{\mathcal{H}}(\gamma) \leq \min \left \{ \frac{R^2}{\gamma^2} , n + 1 \right\}.
\end{equation*}
\end{lemma}
With the definitions and lemmas outlined above, we have established the necessary theoretical groundwork to analyze the complexity of the hypothesis space and the generalization error bound.
\subsection{Supporting Lemmas}
Building upon the aforementioned definitions and lemmas, we propose the subsequent theorem, which provides an empirical risk bound for a hypothesis \( h \in \mathcal{H} \). This bound is applicable to a dataset comprising both ID and OOD samples, under the condition that \( h \) can perfectly classify samples within the training dataset.
\begin{lemma}\label{tm:serb}
Consider a hypothesis \( h \in \mathcal{H} \) that maps \( \mathcal{X} \) to \( \mathbb{R} \) and possesses margins \( \{\gamma_t\}_{t = 0}^T \) on the dataset \( \mathbf{S} = \{\mathbf{S}_t \}_{t = 0}^T \). Suppose the finite fat-shattering dimension of each margin is bounded by the function \( \kappa_t = \text{fat}(\gamma_t / 8) \), which is continuous from the right. Given two distinct datasets \( \mathbf{S} \) and \( \bar{\mathbf{S}} \) consisting of \( M = N(T + 1) \) synthesized samples and for any \( \delta > 0 \), we obtain
\begin{equation*}
\mathcal{P}^{2M}\left\{ \mathbf{S} \bar{\mathbf{S}}: \exists h \in \mathcal{H}, \mathcal{L}_{\mathbf{S}}[h] = 0, \kappa_{0:T}, \mathcal{L}_{\bar{\mathbf{S}}}[h] > \varepsilon_{\bar{\mathbf{S}}} \right\} < \delta,
\end{equation*}
where
\begin{equation*}
\varepsilon_{\bar{\mathbf{S}}} = \frac{ \log \left( 32M\right) \sum_{t = 0}^T \kappa_t \log (8eM)}{M}  + \frac{1}{M}\log \frac{2}{\delta}.
\end{equation*}
\end{lemma}
\begin{proof}
According to the standard permutation argument~\cite{vapnik2015uniform}, the probability can be bounded by the fixed sequence $\mathbf{S} \bar{\mathbf{S}}$ and its corresponding permuted sequence. For datasets $\mathbf{S}_t = \{\mathbf{x}_{i,t} \}_{i = 1}^N \subseteq \mathbf{S}$ and $\bar{\mathbf{S}}_t = \{\bar{\mathbf{x}}_{i,t} \}_{i = 1}^N \subseteq \bar{\mathbf{S}}$ where $t \in [0,T]$, we define their corresponding datasets with the second component determined by the target value of the first component, i.e.,
\begin{equation*}
\begin{aligned}
\mathbf{Z}_t & = \{ (\mathbf{x}_{i,t}, c_{i,t}) \}_{i = 1}^N, \mathbf{Z} = \bigcup_{t = 0}^T \mathbf{Z}_t, \\
\overline{\mathbf{Z}}_t & = \{ (\overline{\mathbf{x}}_{i,t}, \overline{c}_{i,t}) \}_{i = 1}^N, \widehat{\mathbf{Z}} = \bigcup_{t = 0}^T \widehat{\mathbf{Z}}_t.
\end{aligned}
\end{equation*}
Accordingly, for a hypothesis $h \in \mathcal{H}$, we transform the problem of observing the maximal value taken by a set of functions by considering its corresponding function $\widehat{h} \in \widehat{\mathcal{H}}$ for any $\zeta \geq 1$,
\begin{equation*}
h(\mathbf{x}) \mapsto \hat{h}(\mathbf{x}, c) = (2 \zeta - h(\mathbf{x}))(1 - c) + h(\mathbf{x})c.
\end{equation*}
For $\mathbf{z}_t$, we define that $r_t = \max_{(\mathbf{x}, c) \in \mathbf{z}_t} \widehat{h}(\mathbf{x}_{i,t}, c_i)$. Accordingly, at least $N\varepsilon_{\bar{\mathbf{S}}}$ samples $(\overline{\mathbf{x}}, \overline{c}) \in \overline{\mathbf{z}}_t$ satisfy
\begin{equation*}
\widehat{h}(\bar{\mathbf{x}}, \bar{c}) \geq r_t + 2 \widehat{\gamma}_t.
\end{equation*}
Let $\gamma_t = \min \{ \gamma_t' : \text{fat}(\gamma_t' / 4) \leq \kappa_t \}$, we have $\gamma_t \leq \widehat{\gamma}_t$. Without loss of generality, we assume $\gamma_t = 2 \widehat{\gamma}_t$ and $\zeta = r_t + 2 \widehat{\gamma}_t$ according to \cref{def:func}.
For $t \in [0,T]$, we define the following probability event
\begin{equation*}
\begin{aligned}
& \mathcal{Z}_t =  \left\{ \mathbf{S} \bar{\mathbf{S}}: \exists h \in \mathcal{H}, \mathcal{L}_{\mathbf{S}}[h] = 0, \mathcal{C}_t^1, \mathcal{C}_t^2, \mathcal{C}_t^3 \right\}, \\
& \mathcal{C}_t^1 : r_t = \max_{(\mathbf{x}, c) \in \mathbf{Z}_t} \widehat{h}(\mathbf{x}, c),\\
& \mathcal{C}_t^2 : \zeta = r_t + 2 \widehat{\gamma}_t, \\
& \mathcal{C}_t^3 : \vert \{ (\bar{\mathbf{x}}, \bar{c}) \in \overline{\mathbf{Z}}: \widehat{h}(\bar{\mathbf{x}}, \bar{c}) \geq 2 \widehat{\gamma}_t + r_t \} \vert >  M \varepsilon_{\bar{\mathbf{S}}}.\\
\end{aligned}
\end{equation*}
and the following auxiliary function
\begin{equation*}
\pi_t(\hat{h}) =
\left \{
\begin{array}{ll}
  \zeta & \text{if} \, \hat{h} \geq \zeta\\
  \zeta - 2 \widehat{\gamma}_t & \text{if} \, \hat{h} \leq \zeta - 2 \widehat{\gamma}_t\\
  \hat{h} & \text{otherwise}
\end{array} \right.,
\end{equation*}
and let $\pi_t(\hat{\mathcal{H}}) = \left\{ \pi_t(\hat{h}): \hat{h} \in \mathcal{H}\right\}$. Consider the~\cref{def:cover} and a minimal $\widehat{\gamma}_t$-cover $\mathcal{B}_t$ of $\pi_t(\hat{\mathcal{H}})$, we have that for any $\widehat{h} \in \hat{\mathcal{H}}$, there exists $\widehat{h}_{\mathcal{B}_t} \in \hat{\mathcal{H}}_{\mathcal{B}_t}$, with
\begin{equation*}
\vert \pi_t(\hat{h}(\mathbf{x}, c)) - \pi_t(\widehat{h}_{\mathcal{B}_t}(\mathbf{x}, c)) \vert < \widehat{\gamma}_t, \forall (\mathbf{x}, c) \in \mathbf{Z}_t \cup \bar{\mathbf{Z}}_t.
\end{equation*}
Therefore, according to the definition of $r_t$, for all $(\mathbf{x}, c) \in \mathbf{Z}_t$, we have
\begin{equation*}
\begin{aligned}
\hat{h}(\mathbf{x}, c) \leq r_t &  = \zeta - 2\hat{\gamma}_t, \\
\pi_t(\hat{h}(\mathbf{x}, c)) & = \zeta - 2\hat{\gamma}_t, \\
\pi_t(\widehat{h}_{\mathcal{B}_t}(\mathbf{x}, c)) & \leq \zeta - \hat{\gamma}_t.
\end{aligned}
\end{equation*}
Therefore, there are at least $M\varepsilon_{\bar{\mathbf{S}}}$ samples $(\bar{\mathbf{x}}, \bar{c}) \in \overline{\mathbf{Z}}$ such that
\begin{equation*}
\begin{aligned}
\widehat{h}(\bar{\mathbf{x}}, \bar{c}) & \geq \zeta = r_t + 2 \widehat{\gamma}_t, \\
\pi_t(\widehat{h}_{\mathcal{B}_t}(\bar{\mathbf{x}}, \bar{c}) & \geq \zeta - \hat{\gamma}_t \geq \max_{(\mathbf{x}, c) \in \mathbf{Z}_t} \pi_t(\widehat{h}_{\mathcal{B}_t}(\mathbf{x}, c)).
\end{aligned}
\end{equation*}
Since $\pi_t$ only reduces the separation between output values, we have
\begin{equation*}
\pi_t(\widehat{h}_{\mathcal{B}_t}(\bar{\mathbf{x}}, \bar{c}) > \pi_t(\widehat{h}_{\mathcal{B}_t}(\mathbf{x}, c)), \forall (\mathbf{x}, c) \in \mathbf{Z}_t, \forall (\bar{\mathbf{x}}, \bar{c}) \in \overline{\mathbf{Z}}.
\end{equation*}
According to the permutation argument, there are at most $2^{-M\varepsilon_{\bar{\mathbf{S}}}}$ of sequences obtained by swapping corresponding points satisfying conditions for a fixed $\widehat{h}_{\mathcal{B}_t} \in \hat{\mathcal{H}}_{\mathcal{B}_t}$. This is because the $M\varepsilon_{\bar{\mathbf{S}}}$ points with the largest $\widehat{h}_{\mathcal{B}_t}$ values must remain for the inequality occur. Therefore, for any $h \in \mathcal{H}$, there are at least $\mathcal{B}_t$ hypothesis $\widehat{h}_{\mathcal{B}_t} \in \hat{\mathcal{H}}_{\mathcal{B}_t}$ satisfying the inequality for $\widehat{\gamma}_t$. Every set of points $\widehat{\gamma}_t$-shattered by $\pi_t(\hat{\mathcal{H}})$ can be $\widehat{\gamma}_t$-shattered by $\hat{\mathcal{H}}$, which indicates that $\text{fat}_{\pi_t(\hat{\mathcal{H}})} (\widehat{\gamma}_t) \leq \text{fat}_{\hat{\mathcal{H}}} (\widehat{\gamma}_t)$. Applying \cref{lm:cover} for $\pi_t(\hat{\mathcal{H}}) \in \left[\zeta - 2 \widehat{\gamma}_t , \zeta \right ]$, we obtain
\begin{equation*}
\begin{aligned}
\mathbb{E}_{\mathbf{Z}\overline{\mathbf{Z}}}(\vert \mathcal{B}_t\vert)
= & \mathbb{E}_{\mathbf{Z}\overline{\mathbf{Z}}}(\mathfrak{N}(\widehat{\gamma}_t, \pi_t(\hat{\mathcal{H}}))) \\
\leq & 2 (32M)^{\kappa_t \log \frac{8eM}{\kappa_t}} \leq 2 (32M)^{k_t \log (8eM)}.
\end{aligned}
\end{equation*}
According to the union bound, we have
\begin{equation*}
\begin{aligned}
& \mathcal{P}^{2M}\left\{ \mathbf{S} \bar{\mathbf{S}}: \exists h \in \mathcal{H}, \mathcal{L}_{\mathbf{S}}[h] = 0, \kappa_{0:T}, \mathcal{L}_{\bar{\mathbf{S}}}[h] > \varepsilon_{\bar{\mathbf{S}}} \right\} \\
\leq & \mathcal{P}\left( \bigcup_{t = 0}^T \mathcal{Z}_t \right) \leq \sum_{t = 0}^T \mathcal{P}\left( \mathcal{Z}_t \right) \leq \sum_{t = 0}^T \mathbb{E}_{\mathbf{Z}\overline{\mathbf{Z}}}(\vert \mathcal{B}_t\vert) 2^{-M\varepsilon_{\bar{\mathbf{S}}}}\\
\leq & 2^{-M\varepsilon_{\bar{\mathbf{S}}}} \sum_{t = 0}^T 2 (32M)^{ \kappa_t \log (8eM)} \leq \delta.\\
\end{aligned}
\end{equation*}
According to the convex function properties and Jensen inequality, we have
\begin{equation*}
2^{-M\varepsilon_{\bar{\mathbf{S}}}}  2 (32M)^{ \sum_{t = 0}^T \kappa_t \log (8eM)} \leq \delta
\end{equation*}
Accordingly, the inequality holds if
\begin{equation*}
\varepsilon_{\bar{\mathbf{S}}} =  \frac{ \log \left( 32M \right) \sum_{t = 0}^T \kappa_t \log (8eM)}{M}  + \frac{1}{M}\log \frac{2}{\delta} .
\end{equation*}
\end{proof}
Drawing upon the empirical risk bound delineated in \cref{tm:serb}, we are positioned to derive the expected risk bound. This derivation is pertinent when there is a binary classifier at play, capable of classifying samples from the training dataset flawlessly. That is, achieving a zero empirical risk. The focus here is on establishing a bound on the generalization error, which is accomplished by uniformly bounding the probabilities across all conceivable margins.
\begin{lemma}\label{tm:erb}
Consider a hypothesis space \(\mathcal{H}\) restricted to a ball of radius \(R\). Let \(h \in \mathcal{H}\) be a hypothesis that accurately classifies \(M = N(T + 1)\) samples from \(\mathbf{S} \in \mathcal{P}\), with a margin of \(\gamma_t = \zeta - r_t\) and fat dimension \(\kappa_t = \text{fat}(\gamma_t / 8)\) for each dataset \(\mathbf{S}_t\) assigned to \(K\) classes. Here, \(\zeta \geq 1\), \(r_t = \max_{\mathbf{x} \in \mathbf{S}_t} h(\mathbf{x})\), and \(t \in [0, T]\). With a probability of at least \(1 - \delta\), the generalization error bound is given by:
\begin{equation*}
\mathcal{L}_{\mathcal{P}_{\mathcal{S}}}[h] \leq \frac{620 R^2 \log(32M)}{\sqrt{M^3} \sum_{t = 0}^T (\zeta - r_t)^2} + \frac{9}{\sqrt{N\delta}}.
\end{equation*}
\end{lemma}
\begin{proof}
The uniform convergence bound of the generalization error is defined as
\begin{equation*}
\begin{aligned}
& \mathcal{P}_{\mathbf{S}} \left( \sup_{h \in \mathcal{H}} \mathcal{L}_{\mathbf{S}}[h] - \mathcal{L}_{\mathcal{P}_{\mathcal{S}}}[h] \geq \varepsilon_{\mathcal{P}}\right) \\
\leq & \mathcal{P}_{\mathbf{S}\bar{\mathbf{S}}} \left( \sup_{h \in \mathcal{H}} \vert \mathcal{L}_{\mathbf{S}}[h] - \mathcal{L}_{\bar{\mathbf{S}}}[h]\vert \geq \varepsilon_{\mathcal{P}} / 2\right)\\
\leq & \mathcal{P}^{2M}\left( \bigcup_{\kappa_0 = 1}^{2N} \cdots \bigcup_{\kappa_T = 1}^{2N} J\left(\kappa_{0:T} \right) \right) \\
\leq & \sum_{\kappa_0 = 1}^{2N} \cdots \sum_{\kappa_T = 1}^{2N} \mathcal{P}^{2M} J\left(\kappa_{0:T} \right),
\end{aligned}
\end{equation*}
where $J(\kappa_{0:T})$ is defined as
\begin{equation*}
\begin{aligned}
\mathcal{P}^{2M}\left\{ \mathbf{S} \bar{\mathbf{S}}: \exists h \in \mathcal{H}, \mathcal{L}_{\mathbf{S}}[h] = 0, \kappa_{0:T}, \mathcal{L}_{\bar{\mathbf{S}}}[h] > \varepsilon_{\bar{\mathbf{S}}} \right\}.
\end{aligned}
\end{equation*}
The first inequality arises from \cref{lm:sm}. The second inequality holds since the maximum value of \(\kappa_t (t \in [0, T])\) is \(2N\); specifically, it is impossible to shatter a greater number of points from \(\mathbf{S}_t \cup \overline{\mathbf{S}}_t\). The third inequality is derived from the union bound. Let \(\delta' = \delta / (2N)^{T + 1}\). Then, we have
\begin{equation*}
\mathcal{P}^{2M}\left( J\left(\kappa_{0:T} \right) \right) \leq \delta' / (2N)^{T + 1} = \delta.
\end{equation*}
Applying \cref{lm:fb}, we obtain
\begin{equation*}
\kappa_t < \frac{(8 + \omega) R^2}{(\zeta - r_t)^2} < \frac{66 R^2}{(\zeta - r_t)^2},
\end{equation*}
where \(\omega > 0\) is a small constant ensuring continuity from the right, a condition of this lemma. Without loss of generality, we set \(\omega = 0.1\). Combining the above equations and \cref{tm:serb}, with probability at least \(1 - \delta\), we have
\begin{equation*}
\mathcal{L}_{\mathcal{P}_{\mathcal{S}}}[h] \leq \frac{132 R^2 \log(8eM) \log(32M)}{M \sum_{t = 0}^T (\zeta - r_t)^2} + \frac{6}{N}\log \frac{2N}{\delta}.
\end{equation*}
We complete the proof by applying the Jensen inequality and the following fundamental logarithm inequality to simplify this bound:
\begin{equation*}
\log(x) \leq \frac{x - 1}{\sqrt{x}} \leq \frac{x}{\sqrt{x}}, \quad \forall \, x \geq 1.
\end{equation*}
\end{proof}

\subsection{Proof of \cref{co:a}}
\cref{tm:erb} indicates that the generalization error bound is correlated with the margins over data subsets at different times \( t \), with these margins being determined by the weight function. Consequently, we introduce the weight function \(\alpha(t) = \left(\frac{t}{T}\right)^a\) for \( t \in [0,T] \) to derive a more specific generalization error bound.

Let's apply \(\alpha_t = \left(\frac{t}{T} \right)^a\) as the weight function for \cref{tm:erb}. For a given hypothesis \(h \in \mathcal{H}\) and an input \(\mathbf{x}_t \in \mathbf{S}_t\), the target value is given by
\begin{equation*}
\begin{aligned}
r_t & = \max_{\mathcal{P}_{\theta}(y | \mathbf{x}_{t})} \mathbb{E}_{\mathbf{u} \sim \mathcal{U}} \left[ \max_{y \in [K]} \alpha_t \mathbf{u} + (1 - \alpha_t) \mathcal{P}_{\theta}(y | \mathbf{x}_{t}) \right] \\
& = 1 - \left(1 - \frac{1}{K}\right) \alpha_t.
\end{aligned}
\end{equation*}
Assuming \(\zeta = 1 + \left(1 - \frac{1}{K}\right)\), we can express the sum as
\begin{equation*}
\sum_{t = 0}^T \left(\zeta - r_t \right)^2 = \left(1 - \frac{1}{K}\right)^2  \sum_{t = 0}^T \underbrace{\left( 1 - \left(\frac{t}{T} \right)^a \right)^2}_{\upsilon(t)}.
\end{equation*}
Since \(\upsilon(t)\) is a monotonically decreasing and non-negative function, we can estimate the sum as follows:
\begin{equation*}
\begin{aligned}
\sum_{t = 0}^T \upsilon(t) \geq &  \int_0^{T + 1} \upsilon(t) \, dt \geq  \int_0^{T} \upsilon(t) \, dt = \frac{2Ta^2}{(a + 1)(2a + 1)}.
\end{aligned}
\end{equation*}
The proof is completed by combining the above equations with \cref{tm:erb}.

\begin{table*}\tiny
\centering
\caption{OOD detection results on CIFAR10 and ResNet18. Results are averaged over three random trials.}
\label{tb:C10RN18}
\begin{tabular}{cccccccccccc}
\toprule
\multirow{2}{*}{Training} & \multirow{2}{*}{Method} & \multicolumn{2}{c}{Near-OOD} & \multicolumn{6}{c}{Far-OOD} & \multirow{2}{*}{\begin{tabular}[c]{@{}c@{}}Ave. \\ AUROC\end{tabular}} & \multirow{2}{*}{\begin{tabular}[c]{@{}c@{}}Ave. \\ Rank\end{tabular}} \\
\cmidrule(lr){3-4} \cmidrule(lr){5-10}
 &  & CIFAR100 & TIN & MNIST & SVHN & Textures & Place365 & LSUN & iSUN &  &  \\
\midrule \midrule
\multirow{7}{*}{\begin{tabular}[c]{@{}c@{}}w/o \\ training \\ID data\end{tabular}}
 & MSP         & 87.19$_{\pm0.33}$ & 88.87$_{\pm0.19}$ & 92.63$_{\pm1.57}$ & 91.46$_{\pm0.40}$ & 89.89$_{\pm0.71}$ & 89.92$_{\pm0.47}$ & 93.07$_{\pm0.12}$ & 92.58$_{\pm0.23}$ & 90.70 & 4.56 \\
 & ODIN        & 82.18$_{\pm1.87}$ & 83.55$_{\pm1.84}$ & 95.24$_{\pm1.96}$ & 84.58$_{\pm0.77}$ & 86.94$_{\pm2.26}$ & 85.07$_{\pm1.24}$ & 94.56$_{\pm1.42}$ & 93.90$_{\pm1.39}$ & 88.25 & 4.78 \\
 & EBD         & 86.36$_{\pm0.58}$ & 88.80$_{\pm0.36}$ & 94.32$_{\pm2.53}$ & 91.79$_{\pm0.98}$ & 91.79$_{\pm0.70}$ & 89.47$_{\pm0.78}$ & 92.51$_{\pm0.36}$ & 93.50$_{\pm0.42}$ & 91.07 & 4.11 \\
 & GradNorm    & 54.43$_{\pm1.59}$ & 55.37$_{\pm0.41}$ & 63.72$_{\pm7.37}$ & 53.91$_{\pm6.36}$ & 52.07$_{\pm4.09}$ & 60.50$_{\pm5.33}$ & 59.64$_{\pm4.56}$ & 57.42$_{\pm3.99}$ & 57.13 & 7.00 \\
 & GEN         & 87.21$_{\pm0.36}$ & 89.20$_{\pm0.25}$ & 93.83$_{\pm2.14}$ & 91.97$_{\pm0.66}$ & 90.14$_{\pm0.76}$ & 89.46$_{\pm0.65}$ & 93.17$_{\pm0.62}$ & 93.44$_{\pm0.71}$ & 91.05 & 3.78 \\
 & FeatureNorm & 86.41$_{\pm2.19}$ & 89.30$_{\pm1.53}$ & 94.46$_{\pm1.46}$ & \textbf{98.65}$_{\pm0.69}$ & 92.31$_{\pm1.36}$ & 84.62$_{\pm1.46}$ & 99.96$_{\pm1.33}$ & 95.38$_{\pm1.24}$ & 92.64 & 2.56 \\
 & CA$^-$      & \textbf{88.16}$_{\pm0.26}$ & \textbf{90.12}$_{\pm0.27}$ & \textbf{95.62}$_{\pm0.74}$ & 95.00$_{\pm0.33}$ & \textbf{96.51}$_{\pm0.43}$ & \textbf{90.12}$_{\pm0.45}$ & 97.33$_{\pm0.36}$ & \textbf{96.14}$_{\pm0.33}$ & \textbf{93.63} & \textbf{1.22} \\
\midrule \midrule
\multirow{7}{*}{\begin{tabular}[c]{@{}c@{}}w/ \\ training \\ID data\end{tabular}}
 & G-ODIN      & 88.14$_{\pm0.60}$ & 90.09$_{\pm0.54}$ & 98.95$_{\pm0.53}$ & 97.76$_{\pm0.14}$ & 95.02$_{\pm1.10}$ & 90.31$_{\pm0.65}$ & 93.65$_{\pm0.36}$ & 93.12$_{\pm0.42}$ & 93.38 & 3.67 \\
 & ARPL        & 86.76$_{\pm0.16}$ & 88.12$_{\pm0.14}$ & 92.62$_{\pm0.88}$ & 87.69$_{\pm0.97}$ & 88.57$_{\pm0.43}$ & 88.39$_{\pm0.16}$ & 94.78$_{\pm0.32}$ & 95.64$_{\pm0.29}$ & 90.32 & 5.22 \\
 & MOS         & 70.57$_{\pm3.04}$ & 72.34$_{\pm3.16}$ & 74.81$_{\pm10.05}$ & 73.66$_{\pm9.14}$ & 70.35$_{\pm3.11}$ & 86.81$_{\pm1.85}$ & 74.02$_{\pm2.03}$ & 78.96$_{\pm1.97}$ & 75.19 & 7.00 \\
 & LogitNorm   & 90.95$_{\pm0.22}$ & 93.70$_{\pm0.06}$ & \textbf{99.14}$_{\pm0.45}$ & \textbf{98.25}$_{\pm0.41}$ & 94.77$_{\pm0.43}$ & 94.79$_{\pm0.16}$ & 94.64$_{\pm0.31}$ & 96.31$_{\pm0.27}$ & 95.32 & 2.22 \\
 & CIDER       & 84.47$_{\pm0.19}$ & 91.94$_{\pm0.19}$ & 93.30$_{\pm1.08}$ & 98.06$_{\pm0.07}$ & 93.71$_{\pm0.39}$ & 93.77$_{\pm0.68}$ & 93.03$_{\pm0.21}$ & 94.85$_{\pm0.41}$ & 92.89 & 4.22 \\
 & VIM         & 87.75$_{\pm0.28}$ & 89.62$_{\pm0.33}$ & 94.76$_{\pm0.38}$ & 94.50$_{\pm0.48}$ & 95.15$_{\pm0.36}$ & 89.49$_{\pm0.39}$ & 95.74$_{\pm0.24}$ & 94.50$_{\pm0.35}$ & 92.69 & 4.11 \\
 & CA$^+$      & \textbf{91.21}$_{\pm0.31}$ & \textbf{93.97}$_{\pm0.28}$ & 98.52$_{\pm0.19}$ & 97.41$_{\pm0.36}$ & \textbf{96.21}$_{\pm0.44}$ & \textbf{95.53}$_{\pm0.46}$ & \textbf{96.54}$_{\pm0.34}$ & \textbf{96.81}$_{\pm0.33}$ & \textbf{95.78} & \textbf{1.56} \\
\bottomrule
\end{tabular}
\vspace{0.2cm}
\end{table*}

\begin{table*}\tiny
\centering
\caption{OOD detection results on CIFAR100 and ResNet18. Results are averaged over three random trials.}
\label{tb:C100RN18}
\begin{tabular}{cccccccccccc}
\toprule
\multirow{2}{*}{Training} & \multirow{2}{*}{Method} & \multicolumn{2}{c}{Near-OOD} & \multicolumn{6}{c}{Far-OOD} & \multirow{2}{*}{\begin{tabular}[c]{@{}c@{}}Ave. \\ AUROC\end{tabular}} & \multirow{2}{*}{\begin{tabular}[c]{@{}c@{}}Ave. \\ Rank\end{tabular}} \\
\cmidrule(lr){3-4} \cmidrule(lr){5-10}
 &  & CIFAR10 & TIN & MNIST & SVHN & Textures & Place365 & LSUN & iSUN &  &  \\
\midrule \midrule
\multirow{7}{*}{\begin{tabular}[c]{@{}c@{}}w/o \\ training \\ID data\end{tabular}}
 & MSP         & 78.47$_{\pm0.07}$ & 82.07$_{\pm0.17}$ & 76.08$_{\pm1.86}$ & 78.42$_{\pm0.89}$ & 77.32$_{\pm0.71}$ & 79.22$_{\pm0.29}$ & 83.61$_{\pm0.16}$ & 82.16$_{\pm0.14}$ & 79.67 & 5.22 \\
 & ODIN        & 78.18$_{\pm0.14}$ & 81.63$_{\pm0.08}$ & \textbf{83.79}$_{\pm1.31}$ & 74.54$_{\pm0.76}$ & 79.33$_{\pm1.08}$ & 79.45$_{\pm0.26}$ & 84.10$_{\pm0.12}$ & 82.36$_{\pm0.09}$ & 80.42 & 4.00 \\
 & EBD         & 79.05$_{\pm0.11}$ & 82.76$_{\pm0.08}$ & 79.18$_{\pm1.37}$ & 82.03$_{\pm1.74}$ & 78.35$_{\pm0.83}$ & 79.52$_{\pm0.23}$ & 83.95$_{\pm0.14}$ & 83.66$_{\pm0.21}$ & 81.06 & 3.11 \\
 & GradNorm    & 70.32$_{\pm0.20}$ & 69.95$_{\pm0.79}$ & 65.35$_{\pm1.12}$ & 76.95$_{\pm4.73}$ & 64.58$_{\pm1.13}$ & 69.69$_{\pm0.17}$ & 78.01$_{\pm0.21}$ & 79.64$_{\pm0.71}$ & 71.81 & 6.89 \\
 & GEN         & 79.38$_{\pm0.04}$ & 83.25$_{\pm0.13}$ & 79.29$_{\pm2.05}$ & 81.41$_{\pm1.50}$ & 78.74$_{\pm0.81}$ & 80.28$_{\pm0.27}$ & 84.63$_{\pm0.34}$ & 82.26$_{\pm0.21}$ & 81.16 & 2.67 \\
 & FeatureNorm & 78.12$_{\pm0.23}$ & 79.51$_{\pm0.14}$ & 78.41$_{\pm1.32}$ & 80.69$_{\pm1.06}$ & 79.63$_{\pm0.82}$ & 76.33$_{\pm0.33}$ & 82.65$_{\pm0.42}$ & 83.06$_{\pm0.39}$ & 79.80 & 4.67 \\
 & CA$^-$      & \textbf{80.35}$_{\pm0.21}$ & \textbf{84.21}$_{\pm0.15}$ & 78.21$_{\pm1.44}$ & \textbf{85.99}$_{\pm0.87}$ & \textbf{81.01}$_{\pm0.79}$ & \textbf{81.21}$_{\pm0.29}$ & \textbf{87.56}$_{\pm0.19}$ & \textbf{86.21}$_{\pm0.23}$ & \textbf{83.09} & \textbf{1.44} \\
\midrule \midrule
\multirow{7}{*}{\begin{tabular}[c]{@{}c@{}}w/ \\ training \\ID data\end{tabular}}
 & G-ODIN      & 73.04$_{\pm0.39}$ & 81.26$_{\pm0.29}$ & \textbf{91.15}$_{\pm2.86}$ & 83.74$_{\pm3.10}$ & \textbf{89.62}$_{\pm0.36}$ & 78.17$_{\pm0.62}$ & 84.65$_{\pm0.41}$ & 86.96$_{\pm0.36}$ & 83.57 & 2.78 \\
 & ARPL        & 73.38$_{\pm0.78}$ & 76.50$_{\pm1.11}$ & 73.77$_{\pm5.89}$ & 76.45$_{\pm1.00}$ & 69.63$_{\pm1.33}$ & 74.62$_{\pm0.57}$ & 79.21$_{\pm0.81}$ & 77.74$_{\pm0.62}$ & 75.16 & 6.44 \\
 & MOS         & 78.54$_{\pm0.13}$ & 82.26$_{\pm0.25}$ & 80.68$_{\pm1.65}$ & 81.59$_{\pm3.81}$ & 79.92$_{\pm0.57}$ & 78.50$_{\pm0.34}$ & 83.96$_{\pm0.34}$ & 82.63$_{\pm0.24}$ & 81.01 & 4.22 \\
 & LogitNorm   & 74.57$_{\pm0.39}$ & 82.37$_{\pm0.24}$ & 90.69$_{\pm1.38}$ & 82.80$_{\pm4.57}$ & 72.37$_{\pm0.67}$ & 80.25$_{\pm0.61}$ & 84.51$_{\pm0.51}$ & 83.64$_{\pm0.49}$ & 81.40 & 3.44 \\
 & CIDER       & 67.55$_{\pm0.60}$ & 78.65$_{\pm0.35}$ & 68.14$_{\pm3.98}$ & \textbf{97.17}$_{\pm0.34}$ & 82.21$_{\pm1.93}$ & 74.43$_{\pm0.64}$ & 81.87$_{\pm0.42}$ & 81.14$_{\pm0.56}$ & 78.90 & 5.44 \\
 & VIM         & 72.21$_{\pm0.41}$ & 77.76$_{\pm0.16}$ & 81.89$_{\pm1.02}$ & 83.14$_{\pm3.71}$ & 85.91$_{\pm0.78}$ & 75.85$_{\pm0.37}$ & 85.41$_{\pm0.41}$ & 84.61$_{\pm0.33}$ & 80.85 & 4.11 \\
 & CA$^+$      & \textbf{79.21}$_{\pm0.24}$ & \textbf{82.96}$_{\pm0.15}$ & 88.33$_{\pm0.22}$ & 89.75$_{\pm0.56}$ & 84.51$_{\pm0.54}$ & \textbf{80.82}$_{\pm0.48}$ & \textbf{85.51}$_{\pm0.46}$ & \textbf{87.76}$_{\pm0.32}$ & \textbf{84.86} & \textbf{1.56} \\
\bottomrule
\end{tabular}
\vspace{0.2cm}
\end{table*}

\begin{table*}\scriptsize
\centering
\caption{OOD detection results on ImageNet-1K and ResNet50.}
\label{tb:INRN50}
\begin{tabular}{cccccccccccc}
\toprule
\multirow{2}{*}{Training} & \multirow{2}{*}{Method} & \multicolumn{2}{c}{Near-OOD} & \multicolumn{6}{c}{Far-OOD} & \multirow{2}{*}{\begin{tabular}[c]{@{}c@{}}Ave. \\ AUROC\end{tabular}} & \multirow{2}{*}{\begin{tabular}[c]{@{}c@{}}Ave. \\ Rank\end{tabular}} \\
\cmidrule(lr){3-4} \cmidrule(lr){5-10}
 &  & SSB-hard & NINCO & iNaturalist & Textures & OIO & Place365 & Caltech256 & COCO &  &  \\
\midrule \midrule
\multirow{7}{*}{\begin{tabular}[c]{@{}c@{}}w/o \\ training \\ID data\end{tabular}}
 & MSP         & 72.09 & 79.95 & 88.41 & 82.43 & 84.86 & 79.76 & 81.58 & 87.31 & 82.05 & 5.22 \\
 & ODIN        & 71.74 & 77.77 & 91.17 & 89.00 & 88.23 & 81.78 & 76.00 & 87.83 & 82.94 & 4.89 \\
 & EBD         & 72.08 & 79.70 & 90.63 & 88.70 & 89.06 & 82.86 & 77.10 & 87.65 & 83.47 & 4.11 \\
 & GradNorm    & 71.90 & 74.02 & 93.89 & 92.05 & 84.82 & 78.62 & 83.00 & 86.20 & 83.06 & 5.00 \\
 & GEN         & 72.01 & 81.70 & 92.44 & 87.59 & 89.26 & 80.66 & 82.30 & 85.29 & 83.91 & 4.00 \\
 & FeatureNorm & 71.86 & 77.54 & 95.76 & \textbf{95.39} & 87.44 & 84.99 & 81.96 & 86.47 & 85.18 & 3.67 \\
 & CA$^-$      & \textbf{73.56} & \textbf{82.55} & \textbf{95.88} & 94.56 & \textbf{89.78} & \textbf{85.33} & \textbf{84.77} & \textbf{88.54} & \textbf{86.87} & \textbf{1.11} \\
\midrule \midrule
\multirow{7}{*}{\begin{tabular}[c]{@{}c@{}}w/ \\ training \\ID data\end{tabular}}
& G-ODIN    & 66.46  & 75.08 & 87.92  & 82.88 & 85.72 & 80.30 & 83.51 & 85.08 & 80.87 & 6.25\\
& ARPL      & \textbf{72.67} & 79.92  & 88.48 & 82.70 & 85.32 & 82.12 & 85.20 & 87.74 & 83.02 & 4.25 \\
& MOS       & 69.73  & 75.97 & 94.48 & 71.90 & 81.86  & 80.51 & 86.45 & 85.96 & 80.86 & 5.13 \\
& LogitNorm & 67.50  & 81.73 & 94.57 & 89.30 & 90.75  & 83.11 & 89.39 & 88.79 & 85.64 & 2.63 \\
& CIDER     & 59.34  & 78.60 & 90.76 & 96.38 & 89.39  & 84.51 & 87.63 & 87.61 & 84.28 & 3.88 \\
& VIM       & 65.54  & 78.63 & 89.56 & \textbf{97.97} & 90.50 & 80.87 & 84.30 & 87.33 & 84.34 & 4.38 \\
& CA$^+$    & 68.84  & \textbf{81.78}  & \textbf{95.93} & 95.63  & \textbf{91.43} & \textbf{85.78}  & \textbf{89.96}  & \textbf{89.18} & \textbf{87.32} & \textbf{1.50} \\
\bottomrule
\end{tabular}
\end{table*}

\section{Experiments}\label{sec:experiment}
In this section, we evaluate the effectiveness of our proposed CA approach\footnote{Source code is available at \url{https://github.com/Lawliet-zzl/CA}.} by comparing its performance with state-of-the-art OOD detection methods, both with and without access to training ID data. To address these scenarios, we introduce two variants of the CA algorithm: CA$^-$, which operates independently of training ID data, and CA$^+$, which leverages training ID data to enhance alignment with the ID distribution. Specifically, CA$^-$ applies the regularizer $\mathcal{R}^-(\widehat{\mathbf{x}})$ from \cref{eq:RU} during sample synthesis to incorporate prior knowledge of the synthesized samples, while CA$^+$ uses the regularizer $\mathcal{R}^+(\widehat{\mathbf{x}})$ from \cref{eq:RS} to align synthesized samples more closely with the training ID data distribution. When a distinction between the two versions is unnecessary, we refer to both as CA.

\subsection{Setup}
In this section, we provide a detailed overview of the ID and OOD datasets and network architectures used to evaluate OOD detection performance. We also outline the metrics selected for assessing both ID classification and OOD detection. Finally, we present the implementation details of the proposed CA method.

\subsubsection{Datasets and Network Architectures}
We utilize three ID datasets to train our networks: CIFAR10~\cite{CIFAR10:09}, CIFAR100~\cite{CIFAR10:09}, and ImageNet-1K~\cite{DBLP:conf/cvpr/DengDSLL009}. For evaluating OOD detection capabilities during testing, following well-established practices in OOD detection research~\cite{DBLP:conf/nips/YangWZZDPWCLSDZ22, DBLP:journals/corr/abs-2306-09301, DBLP:conf/nips/TackMJS20, DBLP:conf/cvpr/YuSLJL23}, we consider two types of OOD samples: near-OOD and far-OOD. Near-OOD samples have similar content and semantics to the ID dataset, often containing objects similar to those in the ID set. Far-OOD samples, by contrast, have distinct content, such as digits, textures, or scenes, and differ significantly in semantics. For CIFAR10, we use CIFAR100 and Tiny ImageNet (TIN)~\cite{le2015tiny} as near-OOD datasets and MNIST~\cite{DBLP:journals/spm/Deng12}, SVHN~\cite{netzer2011reading}, Textures~\cite{DBLP:conf/cvpr/CimpoiMKMV14}, Places365~\cite{DBLP:journals/pami/ZhouLKO018}, LSUN~\cite{DBLP:journals/corr/YuZSSX15} and iSUN~\cite{DBLP:journals/corr/XuEZFKX15} as far-OOD datasets. For CIFAR100, the near-OOD datasets are CIFAR10 and TIN, with far-OOD datasets the same as those used for CIFAR10. For ImageNet-1K, the near-OOD datasets are SSB-hard~\cite{DBLP:conf/iclr/Vaze0VZ22} and NINCO~\cite{DBLP:conf/icml/BitterwolfM023}, while the far-OOD datasets are iNaturalist~\cite{DBLP:conf/cvpr/HornASCSSAPB18}, Textures~\cite{DBLP:conf/cvpr/CimpoiMKMV14}, OpenImage-O (OIO)~\cite{wang2022vim}, Place365~\cite{DBLP:journals/pami/ZhouLKO018}, Caltech256~\cite{CAL:06}, and COCO~\cite{DBLP:conf/eccv/LinMBHPRDZ14}. Unless otherwise specified, we use ResNet18~\cite{DBLP:conf/cvpr/HeZRS16} as the network architecture for CIFAR10 and CIFAR100, and ResNet50~\cite{DBLP:conf/cvpr/HeZRS16} for ImageNet-1K.

\subsubsection{Evaluation Metrics}
To assess the OOD detection performance, each method assigns an OOD score to every test sample. We utilize the area under the receiver operating characteristic curve (AUROC)~\cite{AUROC:06} and Detection error~\cite{DBLP:conf/iclr/LiangLS18} as metrics to gauge the ranking efficacy of these scores. Superior OOD detection is reflected by a higher AUROC and a lower Detection error. Specifically, AUROC evaluates the likelihood that an ID sample receives a score higher than an OOD sample. In contrast, Detection pinpoints the proficiency of a model in recognizing OOD samples, with emphasis on minimizing the misclassification of ID samples as OOD. For assessing ID classification prowess, we employ Accuracy, which denotes the fraction of ID samples the model correctly classifies.

\subsubsection{Implementation Details}
To address scenarios with and without training ID data, we introduce two versions of the CA algorithm: CA$^-$ and CA$^+$. Unless otherwise specified, both CA$^-$ and CA$^+$ use $a = 10$ in the weight function and $T = 1000$ for sample synthesis. If not noted, both the standard network and its binary classifier adopt the same architecture. For CA$^-$ with $\mathcal{R}^-(\widehat{\mathbf{x}})$, parameters are $\beta_{\text{TV}} = 10^{-3}$, $\beta_{\text{f}} = 100$, and $\beta_{l_2} = 3\cdot10^{-8}$. For CA$^+$ with $\mathcal{R}^+(\widehat{\mathbf{x}})$, we set $\beta_{\text{MSE}} = 1$. While these parameters perform adequately, the focus of this paper is on utilizing synthesized samples for OOD-sensitive classification, so detailed parameter tuning is beyond the scope of this work.

\begin{table*}\scriptsize
\centering
\caption{OOD detection results on ImageNet-1K and ResNet101.}
\label{tb:INRN101}
\begin{tabular}{cccccccccccc}
\toprule
\multirow{2}{*}{Training} & \multirow{2}{*}{Method} & \multicolumn{2}{c}{Near-OOD} & \multicolumn{6}{c}{Far-OOD} & \multirow{2}{*}{\begin{tabular}[c]{@{}c@{}}Ave. \\ AUROC\end{tabular}} & \multirow{2}{*}{\begin{tabular}[c]{@{}c@{}}Ave. \\ Rank\end{tabular}} \\
\cmidrule(lr){3-4} \cmidrule(lr){5-10}
 &  & SSB-hard & NINCO & iNaturalist & Textures & OIO & Place365 & Caltech256 & COCO &  &  \\
\midrule \midrule
\multirow{7}{*}{\begin{tabular}[c]{@{}c@{}}w/o \\ training \\ID data\end{tabular}}
& MSP & 72.38 & 80.06 & 88.93 & 83.08 & 85.71 & 79.84 & 82.17 & 88.51 & 82.59 & 5.25 \\
& ODIN & 72.64 & 77.70 & 91.75 & 89.32 & 88.70 & 82.23 & 76.37 & 88.06 & 83.35 & 4.75 \\
& EBD & 72.51 & 80.48 & 90.66 & 89.45 & 89.57 & 82.84 & 78.03 & 88.52 & 84.01 & 3.88 \\
& GradNorm & 72.66 & 74.90 & 94.52 & 92.91 & 85.44 & 78.47 & 83.48 & 86.32 & 83.59 & 4.63 \\
& GEN & 72.15 & 82.01 & 93.20 & 88.71 & \textbf{90.31} & 81.15 & 82.47 & 85.42 & 84.43 & 4.25 \\
& FeatureNorm & 71.94 & 78.18 & 96.37 & \textbf{95.57} & 88.44 & 85.04 & 82.17 & 86.89 & 85.57 & 4.00 \\
& CA$^-$ & \textbf{74.15} & \textbf{82.80} & \textbf{97.19} & 94.93 & 90.10 & \textbf{85.99} & \textbf{85.35} & \textbf{89.31} & \textbf{87.48} & \textbf{1.25} \\
\midrule \midrule
\multirow{7}{*}{\begin{tabular}[c]{@{}c@{}}w/ \\ training \\ID data\end{tabular}}
& G-ODIN & 66.48 & 75.77 & 88.43 & 83.71 & 86.52 & 80.79 & 83.69 & 85.18 & 81.32 & 6.25 \\
& ARPL & \textbf{73.51} & 80.28 & 88.87 & 83.67 & 85.37 & 83.07 & 86.16 & 88.53 & 83.68 & 4.25 \\
& MOS & 69.84 & 76.51 & 95.46 & 72.32 & 82.26 & 81.31 & 87.21 & 86.43 & 81.42 & 4.88 \\
& LogitNorm & 68.34 & 81.86 & 94.86 & 89.35 & 91.72 & 83.24 & \textbf{90.34} & 89.48 & 86.15 & 2.50 \\
& CIDER & 59.68 & 79.51 & 91.35 & 97.17 & 89.80 & 84.75 & 88.30 & 88.29 & 84.86 & 3.75 \\
& VIM & 65.97 & 79.00 & 90.13 & \textbf{98.60} & 90.83 & 81.00 & 84.49 & 87.75 & 84.72 & 4.63 \\
& CA$^+$ & 69.63 & \textbf{81.96} & \textbf{96.55} & 96.14 & \textbf{92.07} & \textbf{86.31} & 90.25 & \textbf{89.84} & \textbf{87.84} & \textbf{1.75} \\
\bottomrule
\end{tabular}
\end{table*}

\begin{table*}\scriptsize
\centering
\caption{OOD detection results on ImageNet-1K and ViT-L.}
\label{tb:INVIT}
\begin{tabular}{cccccccccccc}
\toprule
\multirow{2}{*}{Training} & \multirow{2}{*}{Method} & \multicolumn{2}{c}{Near-OOD} & \multicolumn{6}{c}{Far-OOD} & \multirow{2}{*}{\begin{tabular}[c]{@{}c@{}}Ave. \\ AUROC\end{tabular}} & \multirow{2}{*}{\begin{tabular}[c]{@{}c@{}}Ave. \\ Rank\end{tabular}} \\
\cmidrule(lr){3-4} \cmidrule(lr){5-10}
 &  & SSB-hard & NINCO & iNaturalist & Textures & OIO & Place365 & Caltech256 & COCO &  &  \\
\midrule \midrule
\multirow{7}{*}{\begin{tabular}[c]{@{}c@{}}w/o \\ training \\ID data\end{tabular}}
& MSP & 69.32 & 78.55 & 88.68 & 85.10 & 85.28 & 80.74 & 81.76 & 88.13 & 82.19 & 4.63 \\
& ODIN & 69.48 & 78.62 & 88.69 & 86.26 & 85.17 & 82.51 & 76.42 & 88.55 & 81.96 & 4.13 \\
& EBD & 59.75 & 66.54 & 79.50 & 81.41 & 76.99 & 83.20 & 77.19 & 88.62 & 76.65 & 5.00 \\
& GradNorm & 42.88 & 35.62 & 42.74 & 45.11 & 38.35 & 59.20 & 53.60 & 56.73 & 46.78 & 7.00 \\
& GEN & 70.22 & \textbf{83.48} & 94.49 & 90.72 & \textbf{90.48} & 80.77 & 82.77 & 85.62 & 84.82 & 3.00 \\
& FeatureNorm & 72.03 & 77.71 & 96.18 & \textbf{96.27} & 88.03 & 85.90 & 82.66 & 86.58 & 85.67 & 2.88 \\
& CA$^-$ & \textbf{73.95} & 82.66 & \textbf{96.86} & 95.23 & 90.01 & \textbf{86.21} & \textbf{85.47} & \textbf{89.15} & \textbf{87.44} & \textbf{1.38} \\
\midrule \midrule
\multirow{7}{*}{\begin{tabular}[c]{@{}c@{}}w/ \\ training \\ID data\end{tabular}}
& G-ODIN & 67.29 & 75.45 & 88.22 & 83.07 & 86.10 & 81.12 & 84.15 & 85.86 & 81.41 & 6.17 \\
& ARPL & \textbf{73.47} & 80.12 & 89.18 & 83.07 & 85.90 & 82.38 & 85.23 & 88.16 & 83.44 & 4.28 \\
& MOS & 69.79 & 76.46 & 95.15 & 72.36 & 82.11 & 81.10 & 86.52 & 86.05 & 81.19 & 5.33 \\
& LogitNorm & 67.90 & \textbf{82.07} & 95.11 & 90.28 & 91.04 & 83.13 & 89.71 & 89.06 & 86.04 & 2.56 \\
& CIDER & 59.87 & 79.55 & 91.46 & 96.54 & 90.01 & 84.94 & 88.16 & 87.76 & 84.78 & 3.83 \\
& VIM & 65.96 & 79.55 & 90.23 & \textbf{98.83} & 90.77 & 81.18 & 84.95 & 87.61 & 84.88 & 4.28 \\
& CA$^+$ & 69.50 & 81.83 & \textbf{96.11} & 96.27 & \textbf{92.25} & \textbf{85.94} & \textbf{90.37} & \textbf{89.62} & \textbf{87.74} & \textbf{1.56} \\
\bottomrule
\end{tabular}
\end{table*}

\subsection{Comparative Experiments}
To validate the efficacy of our proposed CA method, we benchmark it against leading OOD detection techniques in scenarios with and without access to training ID data. For a standard network trained using an ID dataset, CA learns its binary classifier to discern between ID and OOD samples under both settings. To ensure fairness, when training ID data is absent, we compare CA$^-$ against renowned OOD detection techniques that do not require retaining or fine-tuning the standard network. These include Maximum over Softmax Probability (MSP)~\cite{DBLP:conf/iclr/HendrycksG17}, ODIN~\cite{DBLP:conf/iclr/LiangLS18}, Energy-based Detector (EBD)~\cite{DBLP:conf/nips/LiuWOL20}, GradNorm~\cite{DBLP:conf/nips/HuangGL21}, GEN~\cite{DBLP:conf/cvpr/LiuLZ23}, and FeatureNorm~\cite{DBLP:conf/cvpr/YuSLJL23}. Conversely, when training ID data is available, CA$^+$ is compared to state-of-the-art methods that necessitate retraining the standard network on the ID training data, such as G-ODIN~\cite{DBLP:conf/cvpr/HsuSJK20}, Adversarial Reciprocal Points Learning (ARPL)~\cite{DBLP:journals/pami/ChenPWT22}, Minimum Others Score (MOS)~\cite{huang2021mos}, LogitNorm~\cite{DBLP:journals/corr/abs-2106-09022}, CIDER~\cite{DBLP:conf/iclr/MingSD023}, and ViM~\cite{wang2022vim}. Following standard practices in OOD detection research~\cite{DBLP:conf/nips/TackMJS20}, we train a ResNet18 on CIFAR10/100 for 100 epochs with cross-entropy loss. We use SGD with 0.9 momentum, a learning rate of 0.1 with cosine annealing, and a weight decay of 0.0005, with a batch size of 128. For methods with specific configurations, we apply their official implementations and settings whenever possible. For ImageNet-1K, we use the pre-trained standard network from torchvision, focusing on ResNet50. If official checkpoints are unavailable, we fine-tune the standard network for 30 epochs with a learning rate of 0.001.

The OOD detection performance on CIFAR10, CIFAR100, and ImageNet-1K are presented in \cref{tb:C10RN18}, \cref{tb:C100RN18}, and \cref{tb:INRN50}, respectively. These results highlight the robustness of CA$^-$ in scenarios without access to ID data and the optimal performance of CA$^+$ when ID data is available, both achieving top ranks across different benchmarks, validating the effectiveness of our proposed method. The performance gains can be attributed to the ability of CA to fine-tune the confidence distributions of ID and OOD samples effectively. This approach enables CA to better separate ID and OOD distributions, improving detection accuracy across diverse datasets. While the performance gains on ImageNet-1K are slightly less pronounced compared to CIFAR100, this can be attributed to the increased diversity and complexity inherent in the ImageNet-1K dataset, which contains a broader array of semantic classes and a higher degree of inter-class variability. Such characteristics can challenge most OOD detection methods due to the finer granularity of distinctions required among classes.

Specifically, for the setting without training ID data, CA$^-$ outperforms the closest competitor, FeatureNorm on CIFAR10, by $1.5\%$. CA$^-$ also ranks first with an average rank of $1.22$ across the tested datasets, significantly surpassing the baseline methods in overall ranking. On CIFAR100, CA$^-$ leads the second-best method, GEN, by $1.9\%$, and ranks first with an average rank of $1.44$. On ImageNet-1K, CA$^-$ improves over FeatureNorm by approximately $1.7\%$ and again ranks first with an average rank of $1.11$. In the setting with training ID data, CA$^+$ further demonstrates strong performance. On CIFAR10, CA$^+$ outperforms the next-best method, LogitNorm, by $2.4\%$ and ranking first with an average rank of $1.56$. On CIFAR100, CA$^+$ surpasses G-ODIN by $1.8\%$ and secures a top average rank of $1.56$. On ImageNet-1K, CA$^+$ exceeds LogitNorm by $1.7\%$ and holding the best average rank of $1.50$. These results highlight the robustness of CA$^-$ in scenarios without access to ID data and the optimal performance of CA$^+$ when ID data is available, both achieving top ranks across different benchmarks, which validates the effectiveness of our proposed method.

\begin{table}[]
\centering
\caption{Comparison with methods utilizing auxiliary OOD data. The results represent the average AUROC across eight OOD datasets.}
\label{tb:aood}
\begin{tabular}{cccc}
\toprule
Dataset & CIFAR10  & CIFAR100 & ImageNet-1K \\
Network & ResNet18 & ResNet18 & ResNet50   \\
\midrule
DOE  & 91.11    & 77.89    & 78.10   \\
DOS  & 92.24    & 81.45    & 82.77  \\
GReg & 92.42    & 82.07    & 82.85  \\
\midrule
CA$^-$  & 93.63    & 83.47    & 83.46   \\
CA$^+$  & \textbf{96.02}    & \textbf{85.23}    & \textbf{87.84}   \\                    
\bottomrule
\end{tabular}
\end{table}

\subsection{Auxiliary OOD Supervision}
To further validate the effectiveness of our method, we compare it against recent OOD detection approaches that leverage auxiliary OOD data, including DOE~\cite{DOE:23}, DOS~\cite{DOS:24}, and GReg~\cite{GR:24}. These methods require access to explicitly provided outlier samples drawn from distributions disjoint with the in-distribution training data. In contrast, our method does not require such auxiliary OOD datasets. For fair comparison, we assume that the same standard OOD distribution used in our synthesis pipeline serves as the auxiliary OOD source for the methods. As shown in~\cref{tb:aood}, our method (both CA$^-$ and CA$^+$) consistently outperforms these baselines across all datasets. Notably, CA$^+$ achieves an average improvement of $3.25\%$ AUROC over the strongest baseline GReg. This indicates that while prior methods passively consume auxiliary OOD samples, our approach more effectively leverages the same distribution by synthesizing informative trajectories and calibrating model confidence, i.e., unlocking richer OOD sensitivity from the same supervision source.


\begin{table}\scriptsize
\centering
\caption{OOD detection results across different network architectures trained on ImageNet-1K. Standard network and OOD-sensitive binary classifier both use the same network architecture. The results represent the average AUROC across eight OOD datasets.}
\label{tb:INALL}
\begin{tabular}{cc|ccc|c}
\toprule
Training                                                                     & Method      & \begin{tabular}[c]{@{}c@{}}ResNet50\\ 25.6M\end{tabular} & \begin{tabular}[c]{@{}c@{}}ResNet101\\ 44.5M\end{tabular} & \begin{tabular}[c]{@{}c@{}}ViT-L\\ 307M\end{tabular} & \begin{tabular}[c]{@{}c@{}}Ave.\\ AUROC\end{tabular} \\
\midrule \midrule
\multirow{7}{*}{\begin{tabular}[c]{@{}c@{}}w/o\\ training \\ID data\end{tabular}}
& MSP         & 82.05 & \textbf{82.59} & 82.20 & 82.28          \\
& ODIN        & 82.94 & \textbf{83.35} & 81.96 & 82.75          \\
& EBD         & 83.47 & \textbf{84.01} & 76.65 & 81.38          \\
& GradNorm    & 83.06 & \textbf{83.59} & 46.78 & 71.14          \\
& GEN         & 83.91 & 84.43 & \textbf{84.82} & 84.39          \\
& FeatureNorm & 85.18 & 85.58 & \textbf{85.67} & 85.48          \\
& CA$^-$      & 86.87 & 87.48 & \textbf{87.44} & \textbf{87.26} \\
\midrule \midrule
\multirow{7}{*}{\begin{tabular}[c]{@{}c@{}}w/\\ training \\ID data\end{tabular}}
& G-ODIN      & 80.87 & 81.32 & \textbf{81.41} & 81.20          \\
& ARPL        & 83.02 & \textbf{83.68} & 83.44 & 83.38          \\
& MOS         & 80.86 & \textbf{81.42} & 81.19 & 81.16          \\
& LogitNorm   & 85.64 & \textbf{86.15} & 86.04 & 85.94          \\
& CIDER       & 84.28 & \textbf{84.86} & 84.79 & 84.64          \\
& VIM         & 84.34 & 84.72 & \textbf{84.89} & 84.65          \\
& CA$^+$      & 87.32 & \textbf{87.84} & 87.74 & \textbf{87.63} \\
\bottomrule
\end{tabular}
\end{table}

\subsection{Diverse Network Architectures}
We further evaluated the efficacy of the CA method across various network architectures, specifically ResNet50~\cite{DBLP:conf/cvpr/HeZRS16}, ResNet101~\cite{DBLP:conf/cvpr/HeZRS16}, and ViT-L~\cite{DBLP:conf/iclr/DosovitskiyB0WZ21}, trained on ImageNet-1K. The results are reported in \cref{tb:INRN50}, \cref{tb:INRN101}, and \cref{tb:INVIT} for each individual architecture, and in \cref{tb:INALL}, which provides a summary of these results across all architectures. This assessment provides insights into the robustness and adaptability of CA across different model sizes and architectures, including both convolutional neural and transformer-based networks. In the setting without training ID data, CA$^-$ consistently outperforms other methods across all architectures, achieving top ranks and significantly higher AUROC scores, particularly when compared to FeatureNorm and GEN. In the setting with training ID data, CA$^+$ shows similar advantages, consistently surpassing other methods, including LogitNorm, across all architectures. These results confirm the robustness of CA across various network types and sizes, effectively handling both convolutional neural and transformer-based networks. Notably, larger models with more parameters, such as ResNet101 and ViT-L, tend to achieve higher OOD detection performance. This improvement can be attributed to their greater capacity for learning complex feature representations, which enhances the separation between ID and OOD samples, leading to better detection accuracy.

\begin{figure*}
    \centering
    \begin{subfigure}[b]{0.3\linewidth}
        \centering
        \includegraphics[width=1\linewidth]{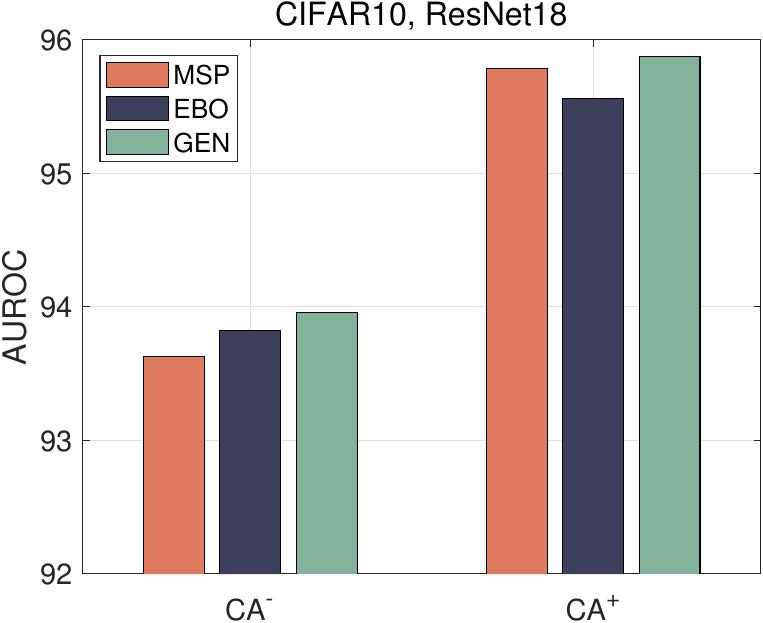}
    \end{subfigure}
    \hspace{0.01\textwidth}
    \begin{subfigure}[b]{0.3\linewidth}
        \centering
        \includegraphics[width=1\linewidth]{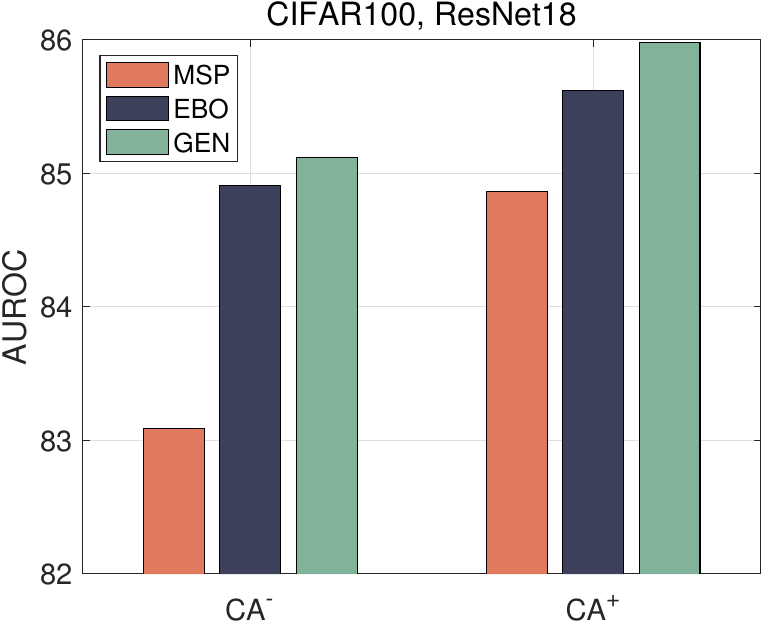}
    \end{subfigure}
    \hspace{0.01\textwidth}
    \begin{subfigure}[b]{0.3\linewidth}
        \centering
        \includegraphics[width=1\linewidth]{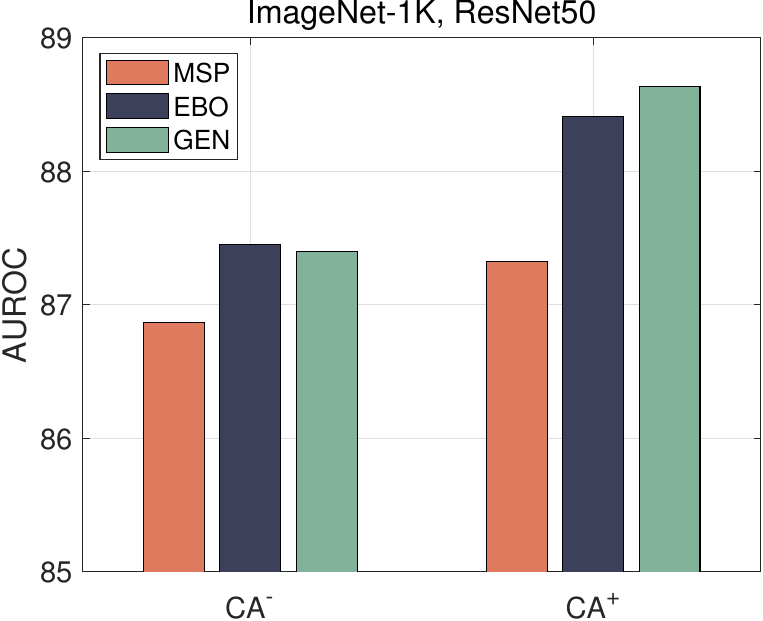}
    \end{subfigure}
    \caption{OOD detection performance using different OOD detectors on CA$^-$ and CA$^+$ with ResNet50 trained on ImageNet-1K. The results represent the average AUROC across eight OOD datasets.}
    \label{fig:detector}
\end{figure*}

\subsection{Diverse OOD Detectors}
To evaluate the robustness of our CA method under different OOD detection techniques, we conducted experiments using three widely recognized OOD detectors: MSP~\cite{DBLP:conf/iclr/HendrycksG17}, EBD~\cite{DBLP:conf/nips/LiuWOL20}, and GEN~\cite{DBLP:conf/cvpr/LiuLZ23}. \cref{fig:detector} presents the OOD detection performance in terms of AUROC for CA$^-$ and CA$^+$. The results demonstrate that CA$^+$ consistently outperforms CA$^-$ across all detectors and datasets. Notably, among the detectors, GEN achieves the highest AUROC scores for both CA$^+$ and CA$^-$, indicating its superior effectiveness in OOD detection within the CA framework. Compared to the baseline MSP, the use of more advanced detectors, such as EBD and GEN, leads to further performance gains, especially with CA$^+$ where the refined training process enhances detector capabilities. These findings highlight the flexibility of CA in integrating various OOD detectors. This adaptability across detectors reinforces the versatility of CA in diverse OOD detection scenarios. The observed improvements are likely due to the capability of CA to calibrate confidence distributions, which complements the strengths of advanced detectors, allowing them to better distinguish between ID and OOD samples, further boosting OOD detection performance.

\begin{table}\scriptsize
  \caption{OOD detection and ID classification using various binary classifier architectures for a standard network pre-trained on ImageNet-1K with ResNet50. The results represent the average AUROC across eight OOD datasets.}
  \centering
  \label{tb:trans}
\begin{tabular}{ccccccccc}
\toprule
\multirow{2}{*}{Type} & \multirow{2}{*}{Architecture} & \multicolumn{2}{c}{CA$^-$} & \multicolumn{2}{c}{CA$^+$} \\ \cmidrule(lr){3-4} \cmidrule(lr){5-6}
&   & AUROC & Accuracy & AUROC & Accuracy \\ \midrule \midrule
\multirow{4}{*}{Deep} &ResNet50       & \textbf{86.9} & \textbf{72.3} & \textbf{87.3} & 77.2 \\
&VGG16          & 68.4 & 56.3 & 84.6 & 69.5 \\
&SENet          & 85.6 & 64.5 & 86.4 & 79.3 \\
&ViT-L          & 82.4 & 65.7 & 86.0 & \textbf{80.5} \\ \midrule
\multirow{1}{*}{Shallow} &MLP            & 58.7 & 5.6 & 71.5 & 13.1 \\
&LeNet          & 63.5 & 12.7 & 79.9 & 15.9 \\ \midrule
Adapter &AIM            & 81.5 & 71.6 & 85.1 & 76.9 \\ \bottomrule
\end{tabular}
\end{table}

\subsection{Transferability Analysis}
We assess the adaptability of synthesized samples generated by a standard network, specifically a ResNet50 trained on ImageNet-1K, by testing them with various binary classifiers of different architectures. These classifiers include deep neural networks (e.g., ResNet50~\cite{DBLP:conf/cvpr/HeZRS16}, VGG16~\cite{DBLP:journals/corr/SimonyanZ14a}, SENet~\cite{DBLP:conf/cvpr/HuSS18}, and ViT-L~\cite{DBLP:conf/iclr/DosovitskiyB0WZ21}), shallow networks (such as MLP~\cite{rumelhart1986learning} and LeNet~\cite{DBLP:journals/pieee/LeCunBBH98}), and the adapter model AIM~\cite{DBLP:conf/iclr/YangZXZC023}, built on the base of the standard network. Detailed performance results are provided in \cref{tb:trans}. Among the deep networks, ResNet50 achieves the best performance, particularly with CA$^+$, underscoring the advantage of deeper architectures. Shallow networks, especially MLP, exhibit notably lower performance, highlighting their limitations in OOD sensitivity. The adapter model, AIM, shows competitive results, effectively bridging the performance gap between shallow and deep models. Overall, these results suggest that synthesized samples from a standard network can effectively train a binary classifier for OOD detection across a variety of architectures. Performance is further enhanced when the standard network and binary classifier share the same architecture, likely due to aligned feature extraction that preserves consistency in the feature space, thereby improving the ability of the classifier to distinguish ID from OOD samples.

\begin{figure}
    \centering
    \begin{subfigure}[b]{0.48\linewidth}
        \centering
        \includegraphics[width=1\linewidth]{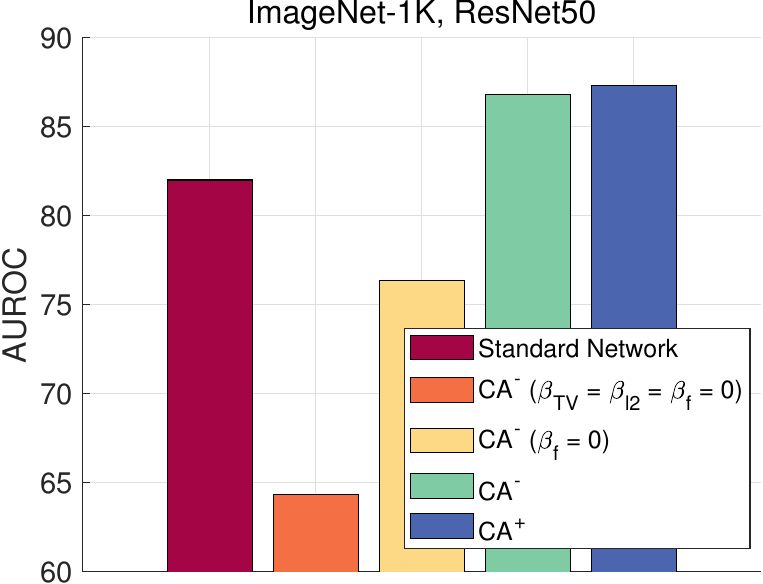}
    \end{subfigure}
    \hspace{0.01\textwidth}
    \begin{subfigure}[b]{0.48\linewidth}
        \centering
        \includegraphics[width=1\linewidth]{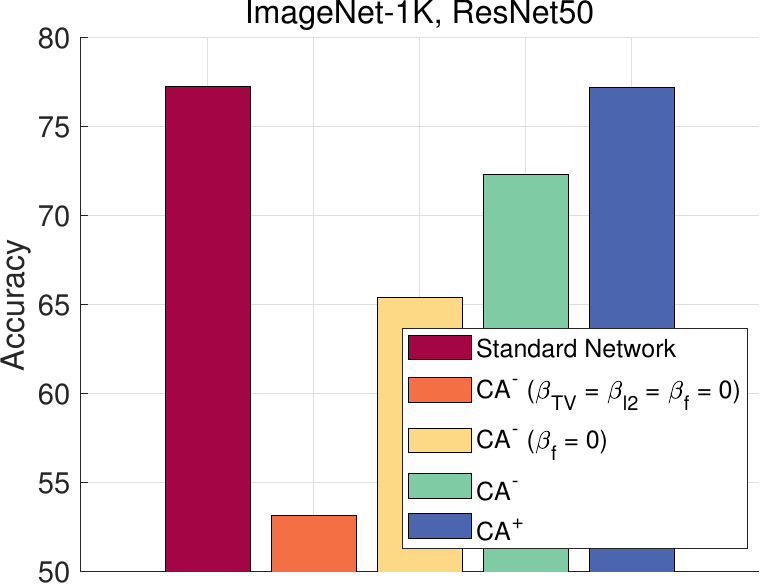}
    \end{subfigure}
    \caption{Ablation study on the effectiveness of synthesizing training data in terms of AUROC and Accuracy. The results represent the average AUROC across eight OOD datasets. Settings where regularization coefficients are zero indicate no use of the corresponding regularizers, while non-zero values indicate its inclusion. CrossEntropy represents the network trained with cross-entropy loss, and all networks apply the MSP detector.}
    \label{fig:aba}
\end{figure}

\subsection{Ablation Studies}
The ablation studies demonstrate that both synthesizing training data with regularization and adjusting predicted label distributions are crucial to enhancing the robustness of the CA method in OOD detection. These techniques enable the model to better differentiate between ID and OOD samples. All networks in these experiments utilize the MSP detector for consistency across evaluations.

\subsubsection{Effect of Synthesizing Training Data}
To evaluate the effectiveness of synthesizing training data, we conducted an ablation study using different configurations of our CA method, as shown in \cref{fig:aba}. The results highlight that CA models incorporating regularization terms achieve higher AUROC and accuracy scores across eight OOD datasets, compared to those trained without these terms. This improvement indicates that synthesizing training data with regularizers enhances the ability of the network to distinguish between ID and OOD samples. The standard network setting serves as a baseline, demonstrating the advantages of incorporating regularization for improved OOD detection.

\subsubsection{Effect of Adjusting Predicted Label Distributions}
To understand the impact of the weight function coefficient \( a \in [0, +\infty) \) in adjusting predicted label distributions, we examine its values across a range: \{0, 0.1, 1, 10, 100\}, as shown in \cref{fig:a}. \(a = 0\) indicates no adjustment, while \(a > 0\) introduces adjustments that gradually increase reliance on the predictions of the network as samples transition from OOD to ID. The results reveal that as \(a\) increases, both CA$^+$ and CA$^-$ models achieve higher AUROC scores. This indicates that adjusting label distributions to reflect confidence levels of ID samples further enhances the robustness of networks in OOD detection, as it aligns OOD samples with low-confidence predictions and ID samples with high-confidence ones.

\begin{figure}
  \centering
  \includegraphics[width=0.65\linewidth]{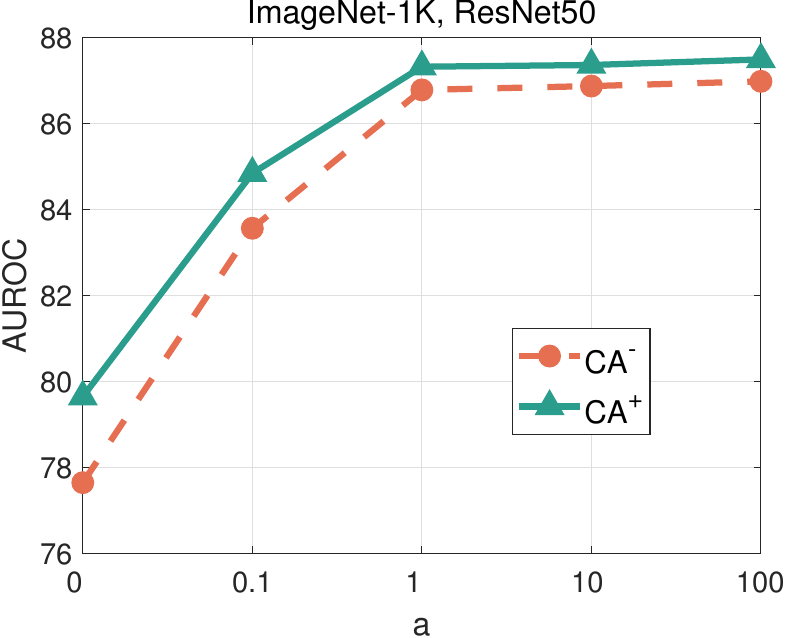}\\
  \caption{Ablation study on the effectiveness of adjusting predicted label distributions. $a = 0$ indicates no adjustment, while $a > 0$ indicates the use of adjusting predicted label distributions.}
  \label{fig:a}
\end{figure}

\begin{figure}
  \centering
  \includegraphics[width=0.65\linewidth]{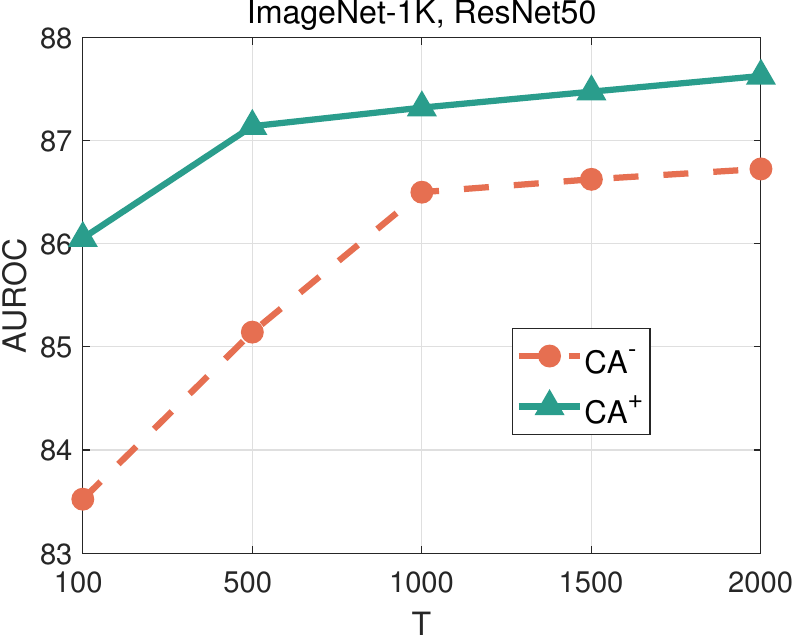}\\
  \caption{OOD detection performance comparison of CA$^-$ and CA$^+$ over varying maximum transition time $T$. The results represent the average AUROC across eight OOD datasets.}
  \label{fig:paraT}
\end{figure}

\begin{figure*}
    \centering
    \begin{subfigure}[b]{0.8\linewidth}
        \centering
        \includegraphics[width=1\linewidth]{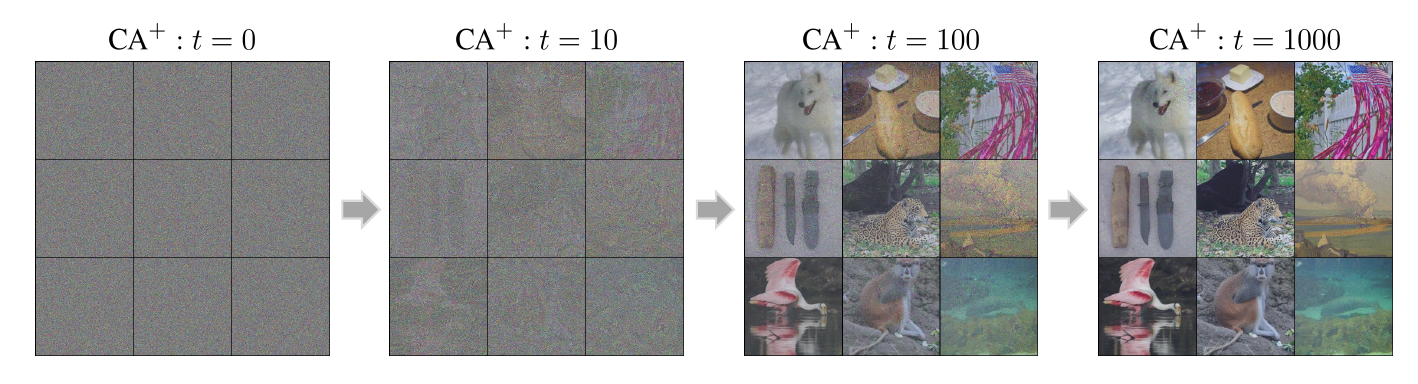}
    \end{subfigure}
    \begin{subfigure}[b]{0.8\linewidth}
        \centering
        \includegraphics[width=1\linewidth]{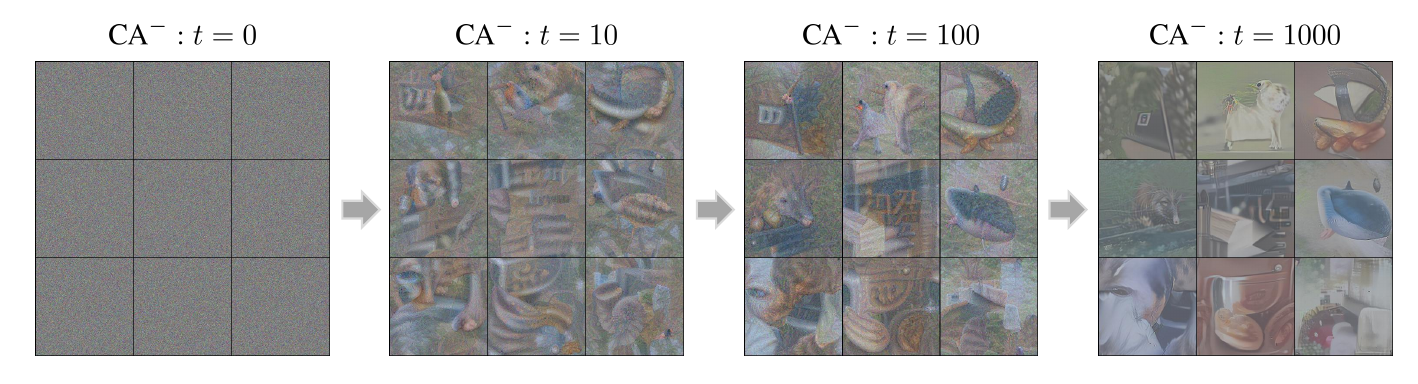}
    \end{subfigure}
    \caption{The image synthesis process from OOD samples to ID samples over time, as generated by CA$^+$ and CA$^-$.}
    \label{fig:progress}
\end{figure*}

\begin{figure*}
    \centering
    \begin{subfigure}[b]{0.4\linewidth}
        \centering
        \includegraphics[width=0.9\linewidth]{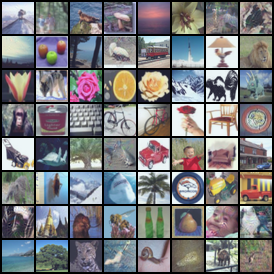}
        \caption{CA$^+$: CIFAR100}
        \label{fig:vis1}
    \end{subfigure}
    \begin{subfigure}[b]{0.4\linewidth}
        \centering
        \includegraphics[width=0.9\linewidth]{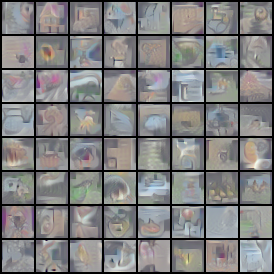}
        \caption{CA$^-$: CIFAR100}
        \label{fig:vis2}
    \end{subfigure}

    \vspace{0.1cm}

    \begin{subfigure}[b]{0.4\linewidth}
        \centering
        \includegraphics[width=0.9\linewidth]{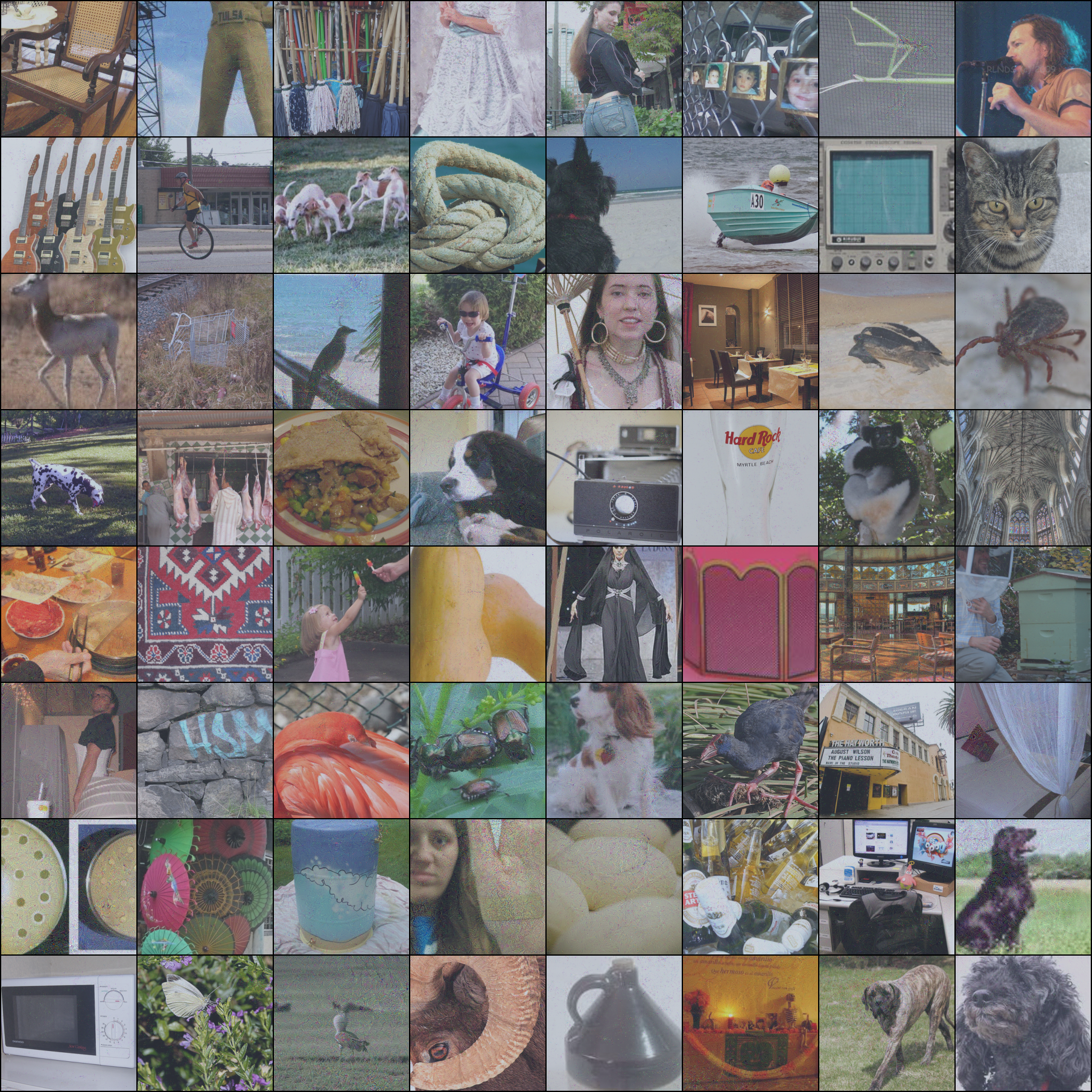}
        \caption{CA$^+$: ImageNet-1K}
        \label{fig:vis3}
    \end{subfigure}
    \begin{subfigure}[b]{0.4\linewidth}
        \centering
        \includegraphics[width=0.9\linewidth]{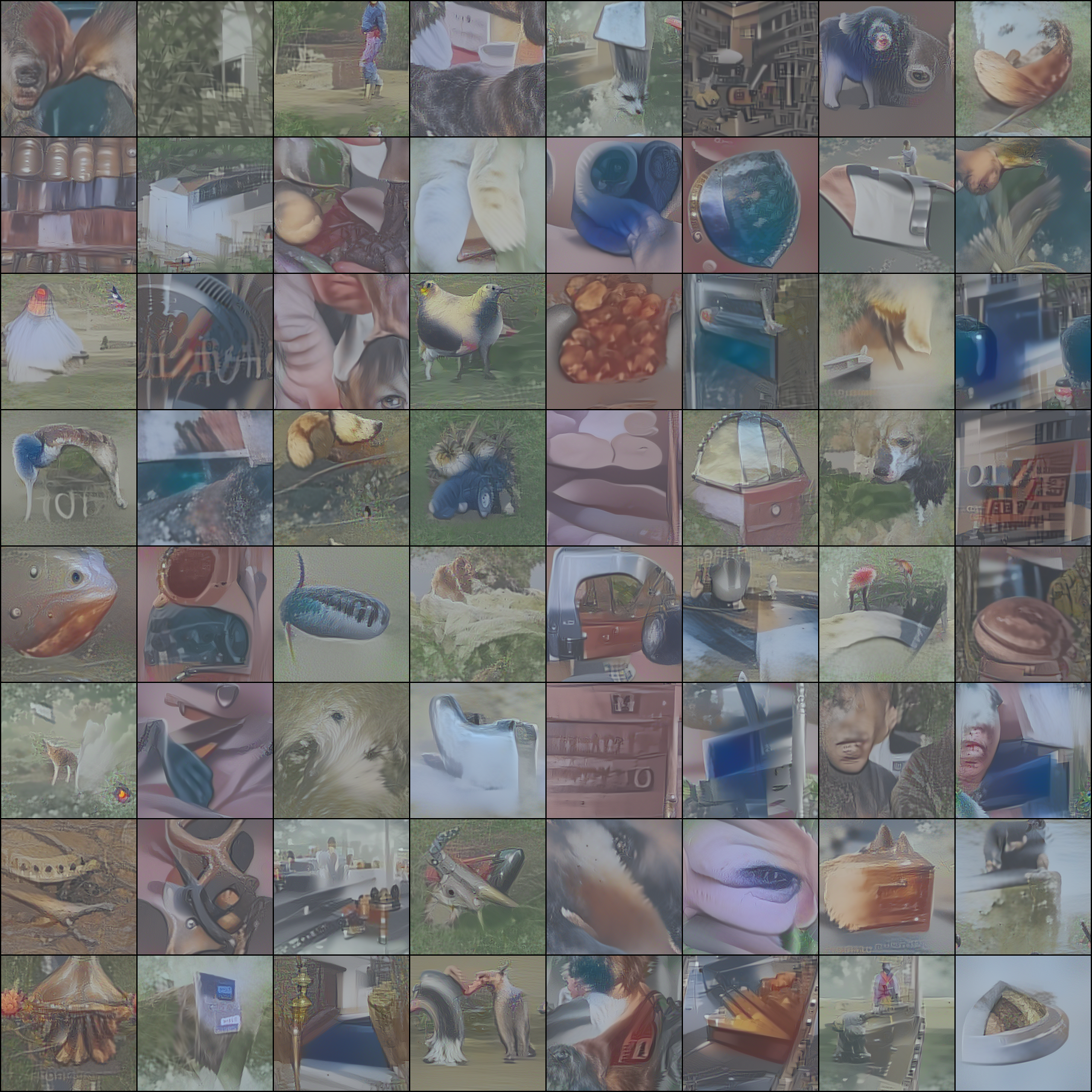}
        \caption{CA$^-$: ImageNet-1K}
        \label{fig:vis4}
    \end{subfigure}
    \caption{Visualization of synthesized images generated by CA$^+$ and CA$^-$ on CIFAR100 and ImageNet-1K.}
    \label{fig:vis}
\end{figure*}

\subsection{Hyperparameter Analysis}
To understand the effect of the maximum transition time $T$ in the sample synthesis phase, we select it from $\{ 100, 500, 1000, 1500, 2000\}$, and the results are presented in \cref{fig:paraT}. CA$^+$ consistently achieves a higher AUROC compared to CA$^-$, suggesting superior performance, and the performance for both methods improves as the maximum transition time increases. This is because, during the sample synthesis phase, OOD samples gradually transform into ID samples. By incorporating real ID samples and increasing the maximum transition time, the synthesized samples at the end are brought closer to the distribution of real ID samples. Thus, the OOD samples evolve towards the ID in a more accurate direction, and this binary classifier can leverage these more accurate samples to learn to differentiate between the two types of samples.

\subsection{Visualization}
In this section, we present visualizations that illustrate the synthesis process and the final generated ID samples for CA$^+$ and CA$^-$. These figures provide insights into how our method gradually transforms OOD samples into ID samples, showing both the time-evolution and final results of the synthesis.

\cref{fig:progress} demonstrates the image synthesis process over time for both CA$^+$ and CA$^-$ on ImageNet-1K. The sequence of images illustrates different stages of the transformation, from initial noise at $t = 0$ to increasingly recognizable features as $t$ progresses, culminating in ID-like samples at $t = 1000$. In the top row, the synthesized images generated by CA$^+$ reveal a clear progression, with specific features becoming more defined over time, resulting in ID samples as $t$ increases. Similarly, in the bottom row, CA$^-$ produces samples that gradually integrate ID characteristics, even without access to training ID data, showcasing the robustness of our method.

\cref{fig:vis} shows the final synthesized ID samples generated by CA$^+$ and CA$^-$ on CIFAR100 and ImageNet-1K. \cref{fig:vis1} and \cref{fig:vis2} display the results for CIFAR100, where CA$^+$ captures more refined details due to its access to training ID data, while CA$^-$ also achieves coherent results. \cref{fig:vis3} and \cref{fig:vis4} present the generated samples for ImageNet-1K, where both CA$^+$ and CA$^-$ produce visually consistent ID images, with CA$^+$ exhibiting slightly sharper details, likely due to the additional ID training information.

\section{Conclusions and Future Work}\label{sec:conclusion}
In this study, we propose a flexible learning framework for OOD knowledge distillation, aimed at enhancing the sensitivity of deep neural networks to OOD samples by training a specialized binary classifier to effectively differentiate between ID and OOD instances. Confidence Amendment (CA) method serves as a specific strategy within this framework, facilitating a structured transition of an OOD sample toward an ID counterpart and emphasizing the incremental establishment of trust in its prediction confidence. These synthesized samples with adjusted predicted label distributions are utilized to train an OOD-sensitive binary classifier. From a theoretical standpoint, the generalization error bound underscores the capability of the classifier in managing unfamiliar ID and OOD samples when paired with a suitable weight function. Comprehensive experiments on various datasets and architectures validate the effectiveness of our method. A promising avenue for future research involves exploring methods that can further transform training ID samples into network-tailored OOD samples to enhance the OOD sensitivity of neural networks.


%


\ifCLASSOPTIONcompsoc
  \section*{Acknowledgments}
\else
  \section*{Acknowledgment}
\fi

This work was supported in part by the Research Start-up Fund for Introduced Talents at Sun Yat-Sen University (No. 67000-12255004), and in part by the Fundamental Research Funds for the Central Universities from the Scientific Research Institute of Sun Yat-Sen University (No. 67000-13130003).


\ifCLASSOPTIONcaptionsoff
  \newpage
\fi



%

\bibliographystyle{IEEEtran}
\bibliography{ref}

\begin{thebibliography}{10}
\providecommand{\url}[1]{#1}
\csname url@samestyle\endcsname
\providecommand{\newblock}{\relax}
\providecommand{\bibinfo}[2]{#2}
\providecommand{\BIBentrySTDinterwordspacing}{\spaceskip=0pt\relax}
\providecommand{\BIBentryALTinterwordstretchfactor}{4}
\providecommand{\BIBentryALTinterwordspacing}{\spaceskip=\fontdimen2\font plus
\BIBentryALTinterwordstretchfactor\fontdimen3\font minus
  \fontdimen4\font\relax}
\providecommand{\BIBforeignlanguage}[2]{{%
\expandafter\ifx\csname l@#1\endcsname\relax
\typeout{** WARNING: IEEEtran.bst: No hyphenation pattern has been}%
\typeout{** loaded for the language `#1'. Using the pattern for}%
\typeout{** the default language instead.}%
\else
\language=\csname l@#1\endcsname
\fi
#2}}
\providecommand{\BIBdecl}{\relax}
\BIBdecl

\bibitem{DBLP:conf/iclr/SuzukiAN20}
T.~Suzuki, H.~Abe, and T.~Nishimura, ``Compression based bound for
  non-compressed network: unified generalization error analysis of large
  compressible deep neural network,'' in \emph{ICLR}, 2020, pp. 1--34.

\bibitem{10136820}
Z.~Zhao, L.~Cao, and K.~Lin, ``Out-of-distribution detection by cross-class
  vicinity distribution of in-distribution data,'' \emph{{IEEE} Trans. Neural
  Networks Learn. Syst.}, pp. 1--12, 2023.

\bibitem{DBLP:conf/nips/YangWZZDPWCLSDZ22}
J.~Yang, P.~Wang, D.~Zou, Z.~Zhou, K.~Ding, W.~Peng, H.~Wang, G.~Chen, B.~Li,
  Y.~Sun, X.~Du, K.~Zhou, W.~Zhang, D.~Hendrycks, Y.~Li, and Z.~Liu, ``Openood:
  Benchmarking generalized out-of-distribution detection,'' in \emph{NeurIPS},
  2022, pp. 1--13.

\bibitem{10271740}
Z.~Zhao, L.~Cao, and K.-Y. Lin, ``Supervision adaptation balancing
  in-distribution generalization and out-of-distribution detection,''
  \emph{{IEEE} Trans. Pattern Anal. Mach. Intell.}, pp. 1--16, 2023.

\bibitem{FIG:23}
Z.~Zhao, L.~Cao, and K.~Lin, ``Revealing the distributional vulnerability of
  discriminators by implicit generators,'' \emph{{IEEE} Trans. Pattern Anal.
  Mach. Intell.}, vol.~45, no.~7, pp. 8888--8901, 2023.

\bibitem{DBLP:journals/tmlr/SalehiMHLRS22}
M.~Salehi, H.~Mirzaei, D.~Hendrycks, Y.~Li, M.~H. Rohban, and M.~Sabokrou, ``A
  unified survey on anomaly, novelty, open-set, and out of-distribution
  detection: Solutions and future challenges,'' \emph{Trans. Mach. Learn.
  Res.}, vol. 2022, pp. 1--81, 2022.

\bibitem{DBLP:conf/cvpr/HeZRS16}
K.~He, X.~Zhang, S.~Ren, and J.~Sun, ``Deep residual learning for image
  recognition,'' in \emph{CVPR}, 2016, pp. 770--778.

\bibitem{CIFAR10:09}
A.~Krizhevsky, G.~Hinton \emph{et~al.}, ``Learning multiple layers of features
  from tiny images,'' Tech. Rep., 2009.

\bibitem{mordvintsev2015deepdream}
A.~Mordvintsev, C.~Olah, and M.~Tyka, ``Inceptionism: Going deeper into neural
  networks,'' Google Research Blog, 2015.

\bibitem{DBLP:conf/cvpr/YinMALMHJK20}
H.~Yin, P.~Molchanov, J.~M. {\'{A}}lvarez, Z.~Li, A.~Mallya, D.~Hoiem, N.~K.
  Jha, and J.~Kautz, ``Dreaming to distill: Data-free knowledge transfer via
  deepinversion,'' in \emph{CVPR}, 2020, pp. 8712--8721.

\bibitem{DBLP:conf/iclr/Allen-ZhuL23}
Z.~Allen{-}Zhu and Y.~Li, ``Towards understanding ensemble, knowledge
  distillation and self-distillation in deep learning,'' in \emph{ICLR}, 2023,
  pp. 1--13.

\bibitem{DBLP:journals/ijcv/GouYMT21}
J.~Gou, B.~Yu, S.~J. Maybank, and D.~Tao, ``Knowledge distillation: {A}
  survey,'' \emph{Int. J. Comput. Vis.}, vol. 129, no.~6, pp. 1789--1819, 2021.

\bibitem{wu-etal-2023-multi}
Q.~Wu, H.~Jiang, H.~Yin, B.~Karlsson, and C.-Y. Lin, ``Multi-level knowledge
  distillation for out-of-distribution detection in text,'' in \emph{ACL},
  2023, pp. 7317--7332.

\bibitem{DBLP:journals/corr/GoodfellowSS14}
I.~J. Goodfellow, J.~Shlens, and C.~Szegedy, ``Explaining and harnessing
  adversarial examples,'' in \emph{ICLR}, 2015, pp. 1--11.

\bibitem{DBLP:conf/icml/Sohl-DicksteinW15}
J.~Sohl{-}Dickstein, E.~A. Weiss, N.~Maheswaranathan, and S.~Ganguli, ``Deep
  unsupervised learning using nonequilibrium thermodynamics,'' in \emph{ICML},
  vol.~37, 2015, pp. 2256--2265.

\bibitem{DBLP:conf/nips/HoJA20}
J.~Ho, A.~Jain, and P.~Abbeel, ``Denoising diffusion probabilistic models,'' in
  \emph{NeurIPS}, 2020, pp. 1--25.

\bibitem{lin2024diversifying}
K.-Y. Lin, J.-R. Du, Y.~Gao, J.~Zhou, and W.-S. Zheng, ``Diversifying
  spatial-temporal perception for video domain generalization,'' in
  \emph{NeurIPS}, 2023, pp. 1--15.

\bibitem{DBLP:journals/corr/KingmaW13}
D.~P. Kingma and M.~Welling, ``Auto-encoding variational bayes,'' in
  \emph{ICLR}, 2014, pp. 1--14.

\bibitem{DBLP:journals/siammax/DuanWWY20}
Y.~Duan, M.~Wang, Z.~Wen, and Y.~Yuan, ``Adaptive low-nonnegative-rank
  approximation for state aggregation of markov chains,'' \emph{{SIAM} J.
  Matrix Anal. Appl.}, vol.~41, no.~1, pp. 244--278, 2020.

\bibitem{DBLP:journals/corr/abs-2110-11334}
J.~Yang, K.~Zhou, Y.~Li, and Z.~Liu, ``Generalized out-of-distribution
  detection: {A} survey,'' \emph{CoRR}, pp. 1--20, 2021.

\bibitem{DBLP:journals/ijcv/YangZL23}
J.~Yang, K.~Zhou, and Z.~Liu, ``Full-spectrum out-of-distribution detection,''
  \emph{Int. J. Comput. Vis.}, vol. 131, no.~10, pp. 2607--2622, 2023.

\bibitem{DBLP:conf/nips/LeeLLS18}
K.~Lee, K.~Lee, H.~Lee, and J.~Shin, ``A simple unified framework for detecting
  out-of-distribution samples and adversarial attacks,'' in \emph{NeurIPS},
  2018, pp. 7167--7177.

\bibitem{DBLP:conf/icml/HendrycksBMZKMS22}
D.~Hendrycks, S.~Basart, M.~Mazeika, A.~Zou, J.~Kwon, M.~Mostajabi,
  J.~Steinhardt, and D.~Song, ``Scaling out-of-distribution detection for
  real-world settings,'' in \emph{ICML}, 2022, pp. 8759--8773.

\bibitem{DBLP:conf/nips/SunGL21}
Y.~Sun, C.~Guo, and Y.~Li, ``{ReAct:} out-of-distribution detection with
  rectified activations,'' in \emph{NeurIPS}, 2021, pp. 144--157.

\bibitem{DBLP:conf/nips/ZhuCXLZ00ZC22}
Y.~Zhu, Y.~Chen, C.~Xie, X.~Li, R.~Zhang, H.~Xue, X.~Tian, B.~Zheng, and
  Y.~Chen, ``Boosting out-of-distribution detection with typical features,'' in
  \emph{NeurIPS}, 2022, pp. 1--12.

\bibitem{DBLP:conf/cvpr/OlberRPSC23}
B.~Olber, K.~Radlak, A.~Popowicz, M.~Szczepankiewicz, and K.~Chachula,
  ``Detection of out-of-distribution samples using binary neuron activation
  patterns,'' in \emph{CVPR}, 2023, pp. 3378--3387.

\bibitem{DBLP:conf/cvpr/AhnPK23}
Y.~H. Ahn, G.~Park, and S.~T. Kim, ``Line: Out-of-distribution detection by
  leveraging important neurons,'' in \emph{CVPR}, 2023, pp. 19\,852--19\,862.

\bibitem{DBLP:conf/iclr/ZhangF0DLWLH023}
J.~Zhang, Q.~Fu, X.~Chen, L.~Du, Z.~Li, G.~Wang, X.~Liu, S.~Han, and D.~Zhang,
  ``Out-of-distribution detection based on in-distribution data patterns
  memorization with modern hopfield energy,'' in \emph{ICLR}, 2023, pp. 1--19.

\bibitem{DBLP:conf/iclr/GomesADP22}
E.~D.~C. Gomes, F.~Alberge, P.~Duhamel, and P.~Piantanida, ``{IGEOOD:} an
  information geometry approach to out-of-distribution detection,'' in
  \emph{ICLR}, 2022, pp. 1--37.

\bibitem{DBLP:conf/icml/ZhuLYLX023}
J.~Zhu, H.~Li, J.~Yao, T.~Liu, J.~Xu, and B.~Han, ``Unleashing mask: Explore
  the intrinsic out-of-distribution detection capability,'' in \emph{ICML},
  2023, pp. 43\,068--43\,104.

\bibitem{DBLP:conf/cvpr/0001AB19}
M.~Hein, M.~Andriushchenko, and J.~Bitterwolf, ``Why relu networks yield
  high-confidence predictions far away from the training data and how to
  mitigate the problem,'' in \emph{CVPR}, 2019, pp. 41--50.

\bibitem{DBLP:conf/cvpr/HsuSJK20}
Y.~Hsu, Y.~Shen, H.~Jin, and Z.~Kira, ``Generalized {ODIN:} detecting
  out-of-distribution image without learning from out-of-distribution data,''
  in \emph{CVPR}, 2020, pp. 10\,948--10\,957.

\bibitem{DBLP:conf/nips/BibasFH21}
K.~Bibas, M.~Feder, and T.~Hassner, ``Single layer predictive normalized
  maximum likelihood for out-of-distribution detection,'' in \emph{NeurIPS},
  2021, pp. 1179--1191.

\bibitem{DBLP:conf/cvpr/CaoZ22}
S.~Cao and Z.~Zhang, ``Deep hybrid models for out-of-distribution detection,''
  in \emph{CVPR}, 2022, pp. 4723--4733.

\bibitem{DBLP:conf/cvpr/0009GLTL022}
X.~Dong, J.~Guo, A.~Li, W.~Ting, C.~Liu, and H.~T. Kung, ``Neural mean
  discrepancy for efficient out-of-distribution detection,'' in \emph{CVPR},
  2022, pp. 19\,195--19\,205.

\bibitem{DOE:23}
J.~Zhu, G.~Yu, J.~Yao, T.~Liu, G.~Niu, M.~Sugiyama, and B.~Han, ``Diversified
  outlier exposure for out-of-distribution detection via informative
  extrapolation,'' in \emph{NeurIPS}, 2023, pp. 1--33.

\bibitem{DOS:24}
W.~Jiang, H.~Cheng, M.~Chen, C.~Wang, and H.~Wei, ``Dos: Diverse outlier
  sampling for out-of-distribution detection,'' in \emph{ICLR}, 2024, pp.
  1--16.

\bibitem{GR:24}
S.~Sharifi, T.~Entesari, B.~Safaei, V.~M. Patel, and M.~Fazlyab,
  ``Gradient-regularized out-of-distribution detection,'' in \emph{ECCV}, 2024,
  pp. 1--24.

\bibitem{DBLP:conf/iclr/HendrycksG17}
D.~Hendrycks and K.~Gimpel, ``A baseline for detecting misclassified and
  out-of-distribution examples in neural networks,'' in \emph{ICLR}, 2017, pp.
  1--12.

\bibitem{DBLP:conf/iclr/LiangLS18}
S.~Liang, Y.~Li, and R.~Srikant, ``Enhancing the reliability of
  out-of-distribution image detection in neural networks,'' in \emph{ICLR},
  2018, pp. 1--27.

\bibitem{DBLP:conf/nips/LiuWOL20}
W.~Liu, X.~Wang, J.~D. Owens, and Y.~Li, ``Energy-based out-of-distribution
  detection,'' in \emph{NeurIPS}, 2020, pp. 1--13.

\bibitem{DBLP:conf/nips/HuangGL21}
R.~Huang, A.~Geng, and Y.~Li, ``On the importance of gradients for detecting
  distributional shifts in the wild,'' in \emph{NeurIPS}, 2021, pp. 677--689.

\bibitem{DBLP:conf/cvpr/LiuLZ23}
X.~Liu, Y.~Lochman, and C.~Zach, ``{GEN:} pushing the limits of softmax-based
  out-of-distribution detection,'' in \emph{CVPR}, 2023, pp. 23\,946--23\,955.

\bibitem{DBLP:conf/cvpr/ZhangX23}
Z.~Zhang and X.~Xiang, ``Decoupling maxlogit for out-of-distribution
  detection,'' in \emph{CVPR}, 2023, pp. 3388--3397.

\bibitem{DBLP:conf/iclr/DjurisicBAL23}
A.~Djurisic, N.~Bozanic, A.~Ashok, and R.~Liu, ``Extremely simple activation
  shaping for out-of-distribution detection,'' in \emph{ICLR}, 2023, pp. 1--22.

\bibitem{DBLP:conf/cvpr/YuSLJL23}
Y.~Yu, S.~Shin, S.~Lee, C.~Jun, and K.~Lee, ``Block selection method for using
  feature norm in out-of-distribution detection,'' in \emph{CVPR}, 2023, pp.
  15\,701--15\,711.

\bibitem{carvalho2024towards}
T.~Carvalho, M.~M. B.~R. Vellasco, and J.~F. Amaral, ``Towards
  out-of-distribution detection using gradient vectors,'' \emph{SSRN Electronic
  Journal}, pp. 1--30, 2024.

\bibitem{DBLP:conf/iclr/LeeLLS18}
K.~Lee, H.~Lee, K.~Lee, and J.~Shin, ``Training confidence-calibrated
  classifiers for detecting out-of-distribution samples,'' in \emph{ICLR},
  2018, pp. 1--16.

\bibitem{huang2021mos}
R.~Huang and Y.~Li, ``{MOS}: Towards scaling out-of-distribution detection for
  large semantic space,'' in \emph{CVPR}, 2021, pp. 8710--8719.

\bibitem{DBLP:conf/nips/HuangWXW022}
W.~Huang, H.~Wang, J.~Xia, C.~Wang, and J.~Zhang, ``Density-driven
  regularization for out-of-distribution detection,'' in \emph{NeurIPS}, 2022,
  pp. 1--14.

\bibitem{DBLP:conf/nips/WangLZZ0L022}
Q.~Wang, F.~Liu, Y.~Zhang, J.~Zhang, C.~Gong, T.~Liu, and B.~Han,
  ``Watermarking for out-of-distribution detection,'' in \emph{NeurIPS}, 2022,
  pp. 1--13.

\bibitem{DBLP:journals/pami/ChenPWT22}
G.~Chen, P.~Peng, X.~Wang, and Y.~Tian, ``Adversarial reciprocal points
  learning for open set recognition,'' \emph{{IEEE} Trans. Pattern Anal. Mach.
  Intell.}, vol.~44, no.~11, pp. 8065--8081, 2022.

\bibitem{wang2022vim}
H.~Wang, Z.~Li, L.~Feng, and W.~Zhang, ``{ViM}: Out-of-distribution with
  virtual-logit matching,'' in \emph{CVPR}, 2022, pp. 4921--4930.

\bibitem{DBLP:conf/iclr/MingSD023}
Y.~Ming, Y.~Sun, O.~Dia, and Y.~Li, ``How to exploit hyperspherical embeddings
  for out-of-distribution detection?'' in \emph{ICLR}, 2023, pp. 1--19.

\bibitem{DBLP:conf/icml/LafonRRT23}
M.~Lafon, E.~Ramzi, C.~Rambour, and N.~Thome, ``Hybrid energy based model in
  the feature space for out-of-distribution detection,'' in \emph{ICML}, vol.
  202, 2023, pp. 18\,250--18\,268.

\bibitem{DBLP:journals/corr/HintonVD15}
G.~E. Hinton, O.~Vinyals, and J.~Dean, ``Distilling the knowledge in a neural
  network,'' \emph{CoRR}, pp. 1--9, 2015.

\bibitem{DBLP:journals/corr/RomeroBKCGB14}
A.~Romero, N.~Ballas, S.~E. Kahou, A.~Chassang, C.~Gatta, and Y.~Bengio,
  ``Fitnets: Hints for thin deep nets,'' in \emph{ICLR}, 2015, pp. 1--13.

\bibitem{DBLP:journals/corr/Lopez-PazBSV15}
D.~Lopez{-}Paz, L.~Bottou, B.~Sch{\"{o}}lkopf, and V.~Vapnik, ``Unifying
  distillation and privileged information,'' in \emph{ICLR}, 2016, pp. 1--10.

\bibitem{DBLP:conf/iccv/ChenW0YLSXX019}
H.~Chen, Y.~Wang, C.~Xu, Z.~Yang, C.~Liu, B.~Shi, C.~Xu, C.~Xu, and Q.~Tian,
  ``Data-free learning of student networks,'' in \emph{ICCV}, 2019, pp.
  3513--3521.

\bibitem{DBLP:journals/corr/abs-1710-07535}
R.~G. Lopes, S.~Fenu, and T.~Starner, ``Data-free knowledge distillation for
  deep neural networks,'' \emph{CoRR}, pp. 1--8, 2017.

\bibitem{DBLP:conf/aaai/HeoLY019}
B.~Heo, M.~Lee, S.~Yun, and J.~Y. Choi, ``Knowledge distillation with
  adversarial samples supporting decision boundary,'' in \emph{AAAI}, 2019, pp.
  3771--3778.

\bibitem{DBLP:conf/nips/GoodfellowPMXWOCB14}
I.~J. Goodfellow, J.~Pouget{-}Abadie, M.~Mirza, B.~Xu, D.~Warde{-}Farley,
  S.~Ozair, A.~C. Courville, and Y.~Bengio, ``Generative adversarial nets,'' in
  \emph{NeurIPS}, 2014, pp. 2672--2680.

\bibitem{DBLP:conf/icml/GalG16}
Y.~Gal and Z.~Ghahramani, ``Dropout as a bayesian approximation: Representing
  model uncertainty in deep learning,'' in \emph{ICML}, vol.~48, 2016, pp.
  1050--1059.

\bibitem{EDFL:94}
M.~J. Kearns and R.~E. Schapire, ``Efficient distribution-free learning of
  probabilistic concepts,'' \emph{J. Comput. Syst. Sci.}, vol.~48, no.~3, pp.
  464--497, 1994.

\bibitem{ML:14}
S.~Shalev-Shwartz and S.~Ben-David, \emph{Understanding Machine Learning From
  Theory to Algorithms}.\hskip 1em plus 0.5em minus 0.4em\relax Cambridge
  University Press, 2014.

\bibitem{SRM:98}
J.~Shawe{-}Taylor, P.~L. Bartlett, R.~C. Williamson, and M.~Anthony,
  ``Structural risk minimization over data-dependent hierarchies,''
  \emph{{IEEE} Trans. Inf. Theory}, vol.~44, no.~5, pp. 1926--1940, 1998.

\bibitem{ED:06}
V.~Vapnik, \emph{Estimation of Dependences Based on Empirical Data}.\hskip 1em
  plus 0.5em minus 0.4em\relax Springer Science \& Business Media, 2006.

\bibitem{GZ:99}
P.~Bartlett and J.~Shawe-Taylor, ``Generalization performance of support vector
  machines and other pattern classifiers,'' \emph{Advances in Kernel
  methods---support vector learning}, pp. 43--54, 1999.

\bibitem{vapnik2015uniform}
V.~N. Vapnik and A.~Y. Chervonenkis, ``On the uniform convergence of relative
  frequencies of events to their probabilities,'' in \emph{Measures of
  complexity: festschrift for alexey chervonenkis}.\hskip 1em plus 0.5em minus
  0.4em\relax Springer, 2015, pp. 11--30.

\bibitem{DBLP:conf/cvpr/DengDSLL009}
J.~Deng, W.~Dong, R.~Socher, L.~Li, K.~Li, and L.~Fei{-}Fei, ``Imagenet: {A}
  large-scale hierarchical image database,'' in \emph{CVPR}, 2009, pp.
  248--255.

\bibitem{DBLP:journals/corr/abs-2306-09301}
J.~Zhang, J.~Yang, P.~Wang, H.~Wang, Y.~Lin, H.~Zhang, Y.~Sun, X.~Du, K.~Zhou,
  W.~Zhang, Y.~Li, Z.~Liu, Y.~Chen, and H.~Li, ``Openood v1.5: Enhanced
  benchmark for out-of-distribution detection,'' \emph{CoRR}, vol.
  abs/2306.09301, pp. 1--18, 2023.

\bibitem{DBLP:conf/nips/TackMJS20}
J.~Tack, S.~Mo, J.~Jeong, and J.~Shin, ``{CSI:} novelty detection via
  contrastive learning on distributionally shifted instances,'' in
  \emph{NeurIPS}, 2020, pp. 1--14.

\bibitem{le2015tiny}
Y.~Le and X.~Yang, ``Tiny imagenet visual recognition challenge,'' \emph{CS
  231N}, vol.~7, no.~7, p.~3, 2015.

\bibitem{DBLP:journals/spm/Deng12}
L.~Deng, ``The {MNIST} database of handwritten digit images for machine
  learning research [best of the web],'' \emph{{IEEE} Signal Process. Mag.},
  vol.~29, no.~6, pp. 141--142, 2012.

\bibitem{netzer2011reading}
Y.~Netzer, T.~Wang, A.~Coates, A.~Bissacco, B.~Wu, and A.~Y. Ng, ``Reading
  digits in natural images with unsupervised feature learning,'' in \emph{NIPS
  Workshop on Deep Learning and Unsupervised Feature Learning}, 2011, pp. 1--9.

\bibitem{DBLP:conf/cvpr/CimpoiMKMV14}
M.~Cimpoi, S.~Maji, I.~Kokkinos, S.~Mohamed, and A.~Vedaldi, ``Describing
  textures in the wild,'' in \emph{CVPR}.\hskip 1em plus 0.5em minus
  0.4em\relax {IEEE} Computer Society, 2014, pp. 3606--3613.

\bibitem{DBLP:journals/pami/ZhouLKO018}
B.~Zhou, {\`{A}}.~Lapedriza, A.~Khosla, A.~Oliva, and A.~Torralba, ``Places:
  {A} 10 million image database for scene recognition,'' \emph{{IEEE} Trans.
  Pattern Anal. Mach. Intell.}, vol.~40, no.~6, pp. 1452--1464, 2018.

\bibitem{DBLP:journals/corr/YuZSSX15}
F.~Yu, Y.~Zhang, S.~Song, A.~Seff, and J.~Xiao, ``{LSUN:} construction of a
  large-scale image dataset using deep learning with humans in the loop,''
  \emph{CoRR}, vol. abs/1506.03365, pp. 1--9, 2015.

\bibitem{DBLP:journals/corr/XuEZFKX15}
P.~Xu, K.~A. Ehinger, Y.~Zhang, A.~Finkelstein, S.~R. Kulkarni, and J.~Xiao,
  ``Turkergaze: Crowdsourcing saliency with webcam based eye tracking,''
  \emph{CoRR}, vol. abs/1504.06755, pp. 1--9, 2015.

\bibitem{DBLP:conf/iclr/Vaze0VZ22}
S.~Vaze, K.~Han, A.~Vedaldi, and A.~Zisserman, ``Open-set recognition: {A} good
  closed-set classifier is all you need,'' in \emph{ICLR}, 2022, pp. 1--26.

\bibitem{DBLP:conf/icml/BitterwolfM023}
J.~Bitterwolf, M.~M{\"{u}}ller, and M.~Hein, ``In or out? fixing imagenet
  out-of-distribution detection evaluation,'' in \emph{ICML}, 2023, pp.
  2471--2506.

\bibitem{DBLP:conf/cvpr/HornASCSSAPB18}
G.~V. Horn, O.~M. Aodha, Y.~Song, Y.~Cui, C.~Sun, A.~Shepard, H.~Adam,
  P.~Perona, and S.~J. Belongie, ``The inaturalist species classification and
  detection dataset,'' in \emph{CVPR}, 2018, pp. 8769--8778.

\bibitem{CAL:06}
G.~Griffin, A.~Holub, and P.~Perona, ``Caltech-256 object category dataset,''
  Tech. Rep., 2007.

\bibitem{DBLP:conf/eccv/LinMBHPRDZ14}
T.~Lin, M.~Maire, S.~J. Belongie, J.~Hays, P.~Perona, D.~Ramanan,
  P.~Doll{\'{a}}r, and C.~L. Zitnick, ``Microsoft {COCO:} common objects in
  context,'' in \emph{ECCV}, 2014, pp. 740--755.

\bibitem{AUROC:06}
J.~Davis and M.~Goadrich, ``The relationship between precision-recall and roc
  curves,'' in \emph{ICML}, 2006, pp. 233--240.

\bibitem{DBLP:journals/corr/abs-2106-09022}
J.~Ren, S.~Fort, J.~Z. Liu, A.~G. Roy, S.~Padhy, and B.~Lakshminarayanan, ``A
  simple fix to mahalanobis distance for improving near-ood detection,''
  \emph{CoRR}, vol. abs/2106.09022, pp. 1--8, 2021.

\bibitem{DBLP:conf/iclr/DosovitskiyB0WZ21}
A.~Dosovitskiy, L.~Beyer, A.~Kolesnikov, D.~Weissenborn, X.~Zhai,
  T.~Unterthiner, M.~Dehghani, M.~Minderer, G.~Heigold, S.~Gelly, J.~Uszkoreit,
  and N.~Houlsby, ``An image is worth 16x16 words: Transformers for image
  recognition at scale,'' in \emph{ICLR}, 2021, pp. 1--21.

\bibitem{DBLP:journals/corr/SimonyanZ14a}
K.~Simonyan and A.~Zisserman, ``Very deep convolutional networks for
  large-scale image recognition,'' in \emph{ICLR}, 2015, pp. 1--14.

\bibitem{DBLP:conf/cvpr/HuSS18}
J.~Hu, L.~Shen, and G.~Sun, ``Squeeze-and-excitation networks,'' in
  \emph{CVPR}, 2018, pp. 7132--7141.

\bibitem{rumelhart1986learning}
D.~E. Rumelhart, G.~E. Hinton, and R.~J. Williams, ``Learning representations
  by back-propagating errors,'' \emph{Nature}, vol. 323, no. 6088, pp.
  533--536, 1986.

\bibitem{DBLP:journals/pieee/LeCunBBH98}
Y.~LeCun, L.~Bottou, Y.~Bengio, and P.~Haffner, ``Gradient-based learning
  applied to document recognition,'' \emph{Proc. {IEEE}}, vol.~86, no.~11, pp.
  2278--2324, 1998.

\bibitem{DBLP:conf/iclr/YangZXZC023}
T.~Yang, Y.~Zhu, Y.~Xie, A.~Zhang, C.~Chen, and M.~Li, ``{AIM:} adapting image
  models for efficient video action recognition,'' in \emph{ICLR}, 2023, pp.
  1--18.

\end{thebibliography}

\begin{IEEEbiography}[{\includegraphics[width=1in,height=1.25in,clip,keepaspectratio]{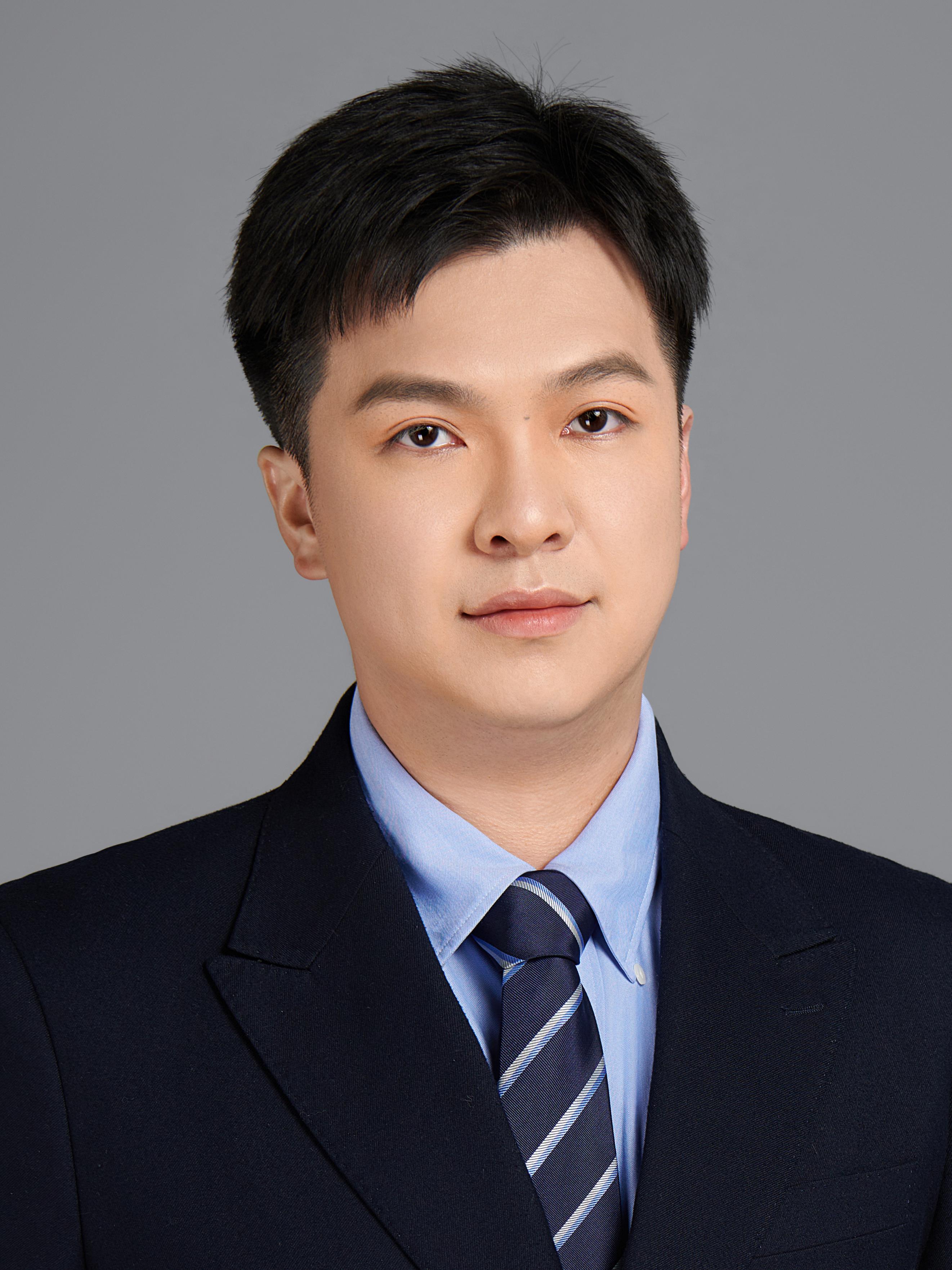}}]{Zhilin Zhao}
received the Ph.D. degree from the University of Technology Sydney in 2022. Prior to that, he received the B.E. and M.E. degrees from the School of Data and Computer Science, Sun Yat-Sen University, China, in 2016 and 2018, respectively. He held postdoctoral research positions at the University of Technology Sydney and Macquarie University, Australia. He is currently an Associate Professor at the School of Computer Science and Engineering, Sun Yat-Sen University. His research interests include generalization analysis, distribution discrepancy estimation, and out-of-distribution detection.
\end{IEEEbiography}

\begin{IEEEbiography}[{\includegraphics[width=1in,height=1.25in,clip,keepaspectratio]{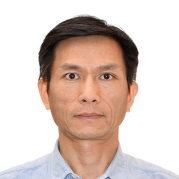}}]{Longbing Cao}(SM'06) received a PhD degree in pattern recognition and intelligent systems at Chinese Academy of Sciences in 2002 and another PhD in computing sciences at University of Technology Sydney in 2005. He is the Distinguished Chair Professor in AI at Macquarie University and an Australian Research Council Future Fellow (professorial level). His research interests include AI and intelligent systems, data science and analytics, machine learning, behavior informatics, and enterprise innovation.
\end{IEEEbiography}

\begin{IEEEbiography}[{\includegraphics[width=1in,height=1.25in,clip,keepaspectratio]{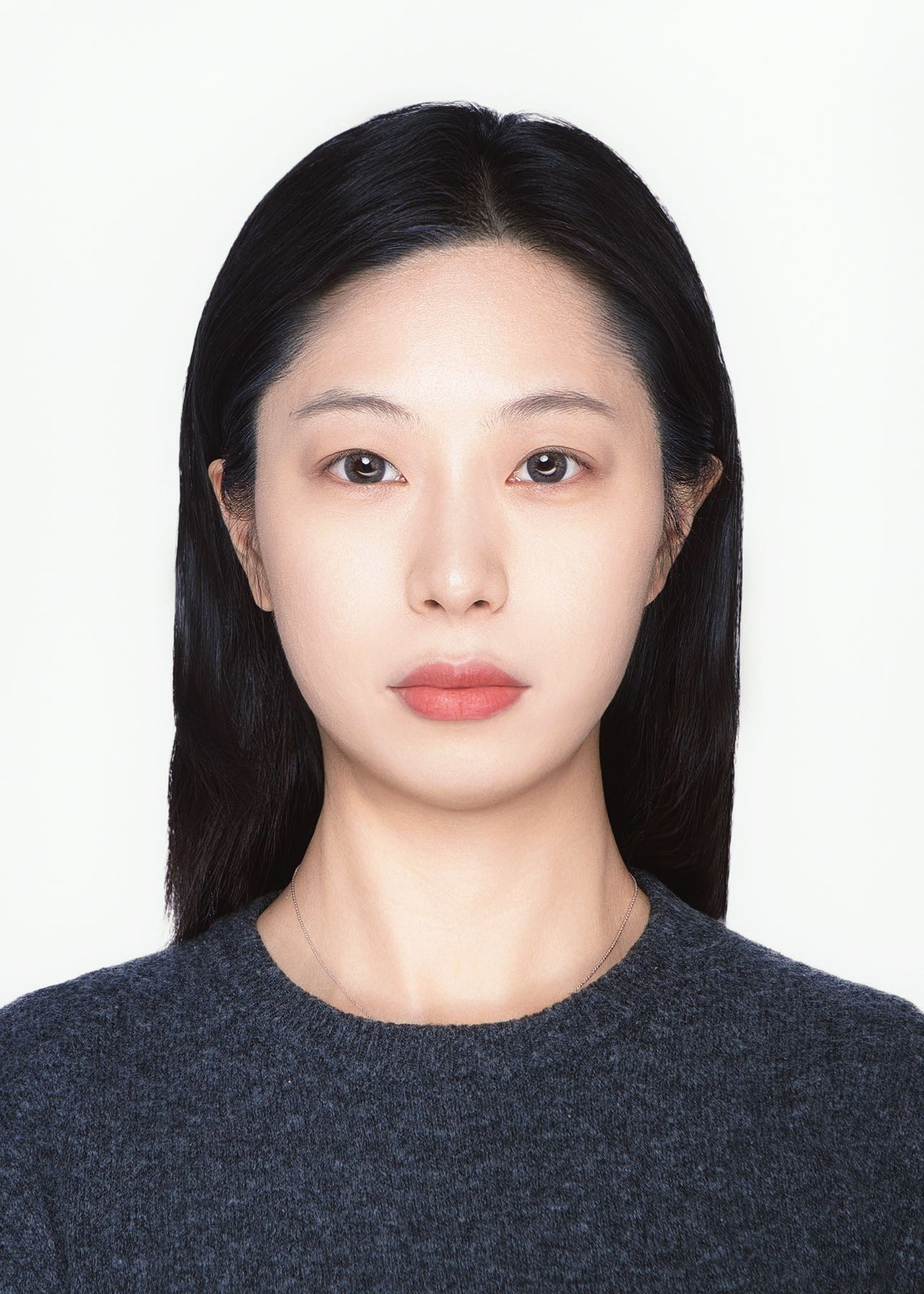}}]{Yixuan Zhang}
received the Ph.D. degree from the University of Technology Sydney in 2023. Prior to that, she received the B.S. from Macquarie University and M.S. degrees from the University of Sydney in 2016 and 2019, respectively. She is currently an assistant professor at Statistics and Data Science, Southeast University, China. Her research interests include data science and machine learning, especially focuses on fairness in machine learning.
\end{IEEEbiography}

\begin{IEEEbiography}[{\includegraphics[width=1in,height=1.25in,clip,keepaspectratio]{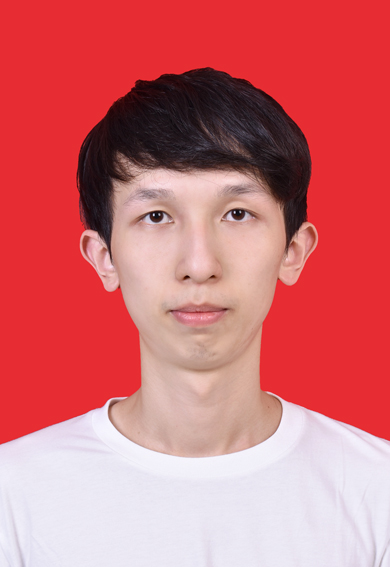}}]
{Kun-Yu Lin} received his B.E., M.E., and Ph.D. degrees from the School of Data and Computer Science, Sun Yat-sen University, China, in 2017, 2019, and 2024, respectively. He is currently a Post-Doctoral Fellow at The University of Hong Kong. His research interests include computer vision and machine learning.
\end{IEEEbiography}

\begin{IEEEbiography}[{\includegraphics[width=1in,height=1.25in,clip,keepaspectratio]{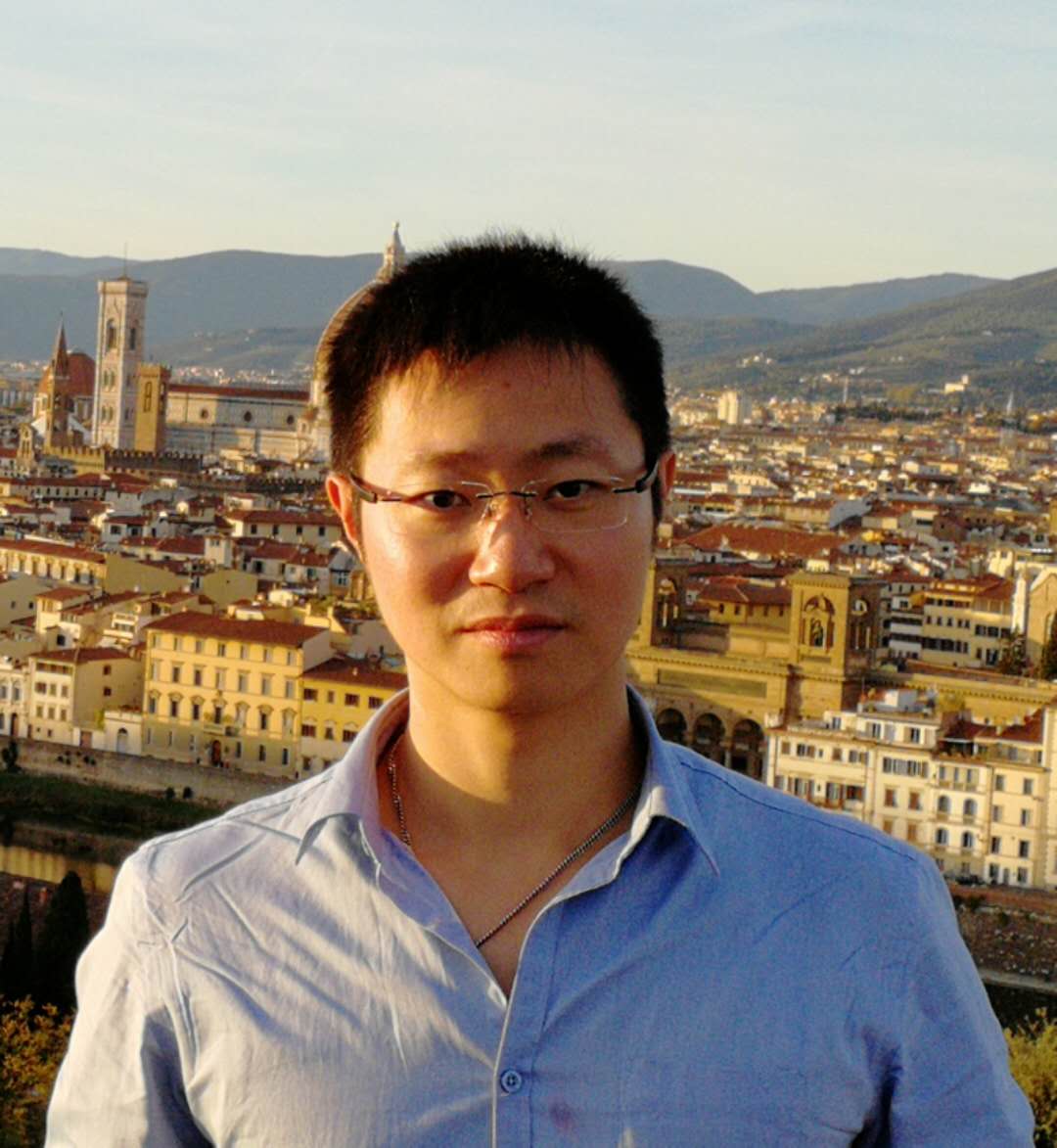}}]
{Wei-Shi Zheng} is now a full Professor with Sun Yat-sen University. His research interests include person/object association and activity understanding, and the related weakly supervised/unsupervised and continuous learning machine learning algorithms. He has now published more than 200 papers, including more than 150 publications in main journals (TPAMI, IJCV, SIGGRAPH, TIP) and top conferences (ICCV, CVPR, ECCV, NeurIPS). He has ever served as area chairs of ICCV, CVPR, ECCV, BMVC, NeurIPS and etc. He is associate editors/on the editorial board of IEEE-TPAMI, Artificial Intelligence Journal, Pattern Recognition. He has ever joined Microsoft Research Asia Young Faculty Visiting Programme. He is a Cheung Kong Scholar Distinguished Professor, a recipient of the Excellent Young Scientists Fund of the National Natural Science Foundation of China, and a recipient of the Royal Society-Newton Advanced Fellowship of the United Kingdom.
\end{IEEEbiography}




\end{document}